\documentclass[10pt,journal,compsoc]{IEEEtran}
\usepackage[nocompress]{cite}
\usepackage[pdftex]{graphicx}
\graphicspath{{figures/}{photos/}}
\ifCLASSINFOpdf

\else

\fi

\usepackage{amsfonts,amssymb}
\usepackage{graphicx}
\usepackage{tcolorbox}
\usepackage{amsmath}
\usepackage{booktabs}
\usepackage{multirow}
\usepackage{caption}
\usepackage{color}
\usepackage{paralist}
\usepackage{enumitem}
\definecolor{mygray}{gray}{.9}

\usepackage{amsmath}
\usepackage{amssymb}
\usepackage{subfigure}

\usepackage{amsfonts,amssymb}
\usepackage{graphicx}

\usepackage{amsmath}
\usepackage{booktabs}
\usepackage{multirow}
\usepackage{color}
\usepackage[switch]{lineno}
\usepackage{enumitem}
\usepackage{url}
\usepackage{xcolor}
\usepackage{colortbl}

\usepackage{graphicx}
\usepackage{amsmath}
\usepackage{amssymb}

\usepackage{bm}
\usepackage[pagebackref=true,breaklinks=true,colorlinks,citecolor=citecolor,bookmarks=false]{hyperref}

\newtheorem{theorem}{Theorem}
\newtheorem{proof}{Proof}

\usepackage[ruled, noline]{algorithm2e}

\newcommand{\pub}[1]{{\color{gray}{\tiny{[{#1}]}}}}

\definecolor{citecolor}{RGB}{34,139,34}
\def \alambic {\includegraphics[width=0.02\linewidth]{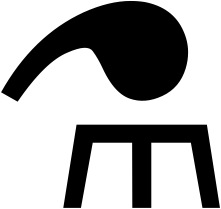}\xspace}

\begin{document}

\title{DNA Family: Boosting Weight-Sharing NAS with Block-Wise Supervisions}

\author{Guangrun Wang$\dag$, Changlin Li$\dag$, Liuchun Yuan, Jiefeng Peng, Xiaoyu Xian,
\\Xiaodan Liang, Xiaojun Chang~\IEEEmembership{Senior~Member,~IEEE}, and Liang Lin$^*$~\IEEEmembership{Senior~Member,~IEEE}
\IEEEcompsocitemizethanks {\IEEEcompsocthanksitem
Guangrun Wang and Changlin Li contribute equally and share the first authorship. G. Wang, L. Yuan, X. Liang, and L. Lin are with Sun Yat-sen University, Guangzhou, China. C. Li and X. Chang are with ReLER lab, AAII, University of Technology Sydney, Australia. X. Xian is with the CRRC Academy, Beijing, China. Email: \{wanggrun, changlinli.ai, ylc0003, jiefengpeng, xdliang328\}@gmail.com; xiaojun.chang@uts.edu.au; xxy@crrc.tech; linliang@ieee.org. Corresponding author: Liang Lin.}}

\markboth{IEEE TRANSACTIONS ON Pattern Analysis and Machine Intelligence}%
{G. Wang\MakeLowercase{\textit{et al.}}: DNA Family: Boosting Weight-Sharing NAS with Block-Wise Supervisions}

\IEEEcompsoctitleabstractindextext{
\begin{abstract}
Neural Architecture Search (NAS), aiming at automatically designing neural architectures by machines, has been considered a key step toward automatic machine learning. One notable NAS branch is the weight-sharing NAS, which significantly improves search efficiency and allows NAS algorithms to run on ordinary computers. Despite receiving high expectations, this category of methods suffers from low search effectiveness. By employing a generalization boundedness tool, we demonstrate that the devil behind this drawback is the untrustworthy architecture rating with the oversized search space of the possible architectures. Addressing this problem, we modularize a large search space into blocks with small search spaces and develop a family of models with the distilling neural architecture (DNA) techniques. These proposed models, namely a DNA family, are capable of resolving multiple dilemmas of the weight-sharing NAS, such as scalability, efficiency, and multi-modal compatibility. Our proposed DNA models can rate all architecture candidates, as opposed to previous works that can only access a sub- search space using heuristic algorithms. Moreover, under a certain computational complexity constraint, our method can seek architectures with different depths and widths. Extensive experimental evaluations show that our models achieve state-of-the-art top-1 accuracy of 78.9\% and 83.6\% on ImageNet for a mobile convolutional network and a small vision transformer, respectively. Additionally, we provide in-depth empirical analysis and insights into neural architecture ratings. Codes available: \url{https://github.com/changlin31/DNA}.
\end{abstract}

\begin{IEEEkeywords}
Block-wise Learning, Neural Architecture Search, Generalization Boundedness, Vision Transformer
\end{IEEEkeywords}}

\maketitle



\begin{figure*}[t]
\centering
\includegraphics[width=0.8\textwidth]{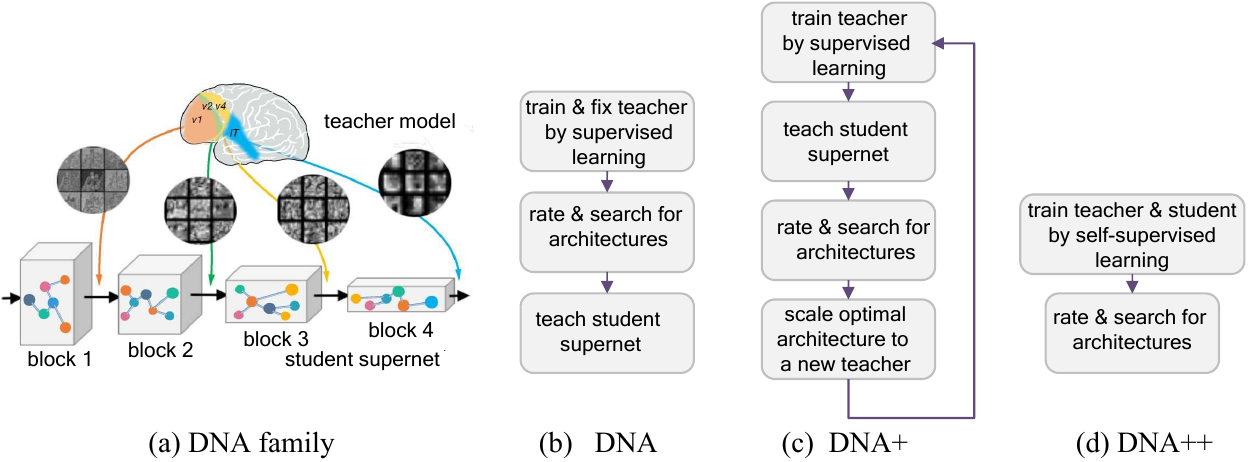}
\vspace{-8pt}
\caption{\small{Illustration of DNA family. (a) Distilling neural architecture technique with block-wise supervision. Architecture candidates (denoted by different nodes and paths) are divided into blocks. (b) Supervised learning (vanilla DNA). (c) Progressive learning (DNA+). (d) Self-supervised learning (DNA++).
}}\label{fig:intro}
\vspace{-14pt}
\end{figure*}

\section{Introduction}\label{sec:intro}

Neural architecture search (NAS)~\cite{RenXCHLCW2020}, aiming to replace human experts with machines in designing neural architectures, is widely anticipated.
Typical works included reinforcement learning approaches \cite{Tan2018MnasNetPN,Tan2019EfficientNetRM}, evolutionary algorithms \cite{chen2018reinforced,negrinho2017deeparchitect}, and Bayesian methods \cite{Kandasamy2018Neural_neurips,White2021BANANAS_aaai}. These methods require multiple trials (i.e., training many architectures separately to assess their quality), which is computationally unaffordable for many researchers. Recent weight-sharing NAS solutions encoded a search space into a weight-sharing supernet and trained all architectures in the supernet at once, significantly improving search efficiency.

Despite high efficiency, the effectiveness of weight-sharing NAS is worrying \cite{sciuto2019evaluating,anonymous2020nas,Peng2021Pi_iccv}. The culprit may be the unreliable architecture rating caused by the oversized search space. Specifically, weight-sharing NAS methods co-train supernet weights to rate architecture candidates. With search space increasing, a generalization boundedness tool suggests that supernet weights have a lower generalization ability, resulting in untrustworthy architecture ratings.
This irrelevance between the predicted and ground-truth architecture ranking makes weight-sharing NAS ineffective.

To boost weight-sharing NAS, reducing search space is helpful. Hence, we factorize the large search space into blocks with small search spaces. 
Specifically, we divide a neural architecture into blocks (see Fig. \ref{fig:intro} (a)) and train each block separately. In this way, the architecture candidates in a block reduce exponentially compared to that in the whole search space, allowing more reliable architecture ratings. With these block-wise representations, we develop three novel NAS models using the distilling neural architecture (DNA) techniques. And we induce these models into a family called the ``DNA Family'', as they share a similar mechanism for building models, i.e., training student supernets by distilling knowledge from a teacher network ~\cite{Hinton2015DistillingTK}. In the following, we summarize the characters of the three DNA models and discuss how they can resolve multiple dilemmas of weight-sharing NAS.

\begin{itemize}
    \item DNA: the vanilla model (illustrated in Fig.~\ref{fig:intro} (b)) employs traditional supervised learning, i.e., casting the latent codes of a fixed teacher network as the supervision. It is very efficient as several student supernets can be trained simultaneously. 
    \item DNA+: this model (illustrated in Fig.~\ref{fig:intro} (c)) allows the teacher network to be progressively updated during the architecture search procedure. The performance of DNA+ is thus less dependent on the initial capacity of the teacher network, leading to better scalability. 
    \item DNA++: this model (illustrated in Fig.~\ref{fig:intro} (d)) makes the teacher network and the student supernets jointly optimize in a self-supervised learning manner, i.e., no labeled data is required during the architecture search. One superiority of DNA++ is the tolerance of different basic architectures of the teacher network and student supernets (e.g., the teacher using a CNN and the student supernets using a vision transformer (ViT)). 
\end{itemize}

With the block-wise representations, we can rate all of the candidate architectures, as opposed to previous works that can only access a sub- search space using heuristic algorithms. 
\textcolor{black}{Our method thus shows a promising level of effectiveness in comparison to other weight-sharing NAS approaches.}
Overall, our contributions are five-fold:\begin{itemize}
\item We show that the excessive search space significantly hampers the effectiveness of most weight-sharing NAS approaches. As the search space expands, supernet weights exhibit reduced generalization ability, leading to imprecise architecture rankings and ineffective searches. To address this issue, we introduce a solution by modularizing the extensive search space using a block-wise representation, which systematically reduces the search scope. 
\item We explore three implementations under the general knowledge distillation framework for balancing the scalability, learning efficiency, and compatibility with various neural network structures. The proposed DNA family can be flexibly adopted for resolving different dilemmas of NAS applications. 
\item Block-wise supervisions enable our models to assess all candidates, which is a departure from previous methods that could only explore a subset of the search space. Additionally, this approach allows us to search for architectures with varying depths and widths within the constraints of computational resources.
\item We highlight the significance of architecture rating through comprehensive empirical studies, encompassing model ranking, evaluations using the Kendall Tau metric, and an assessment of training stability. These studies provide insights into the (in)effectiveness of conventional weight-sharing NAS approaches.
\item \textcolor{black}{Promisingly efficient experimental results} are provided on ImageNet, CIFAR10, and PASCAL VOC segmentation. Typically, our DNA family gets a 78.9\% top-1 accuracy in a mobile setting and an 83.6\% top-1 accuracy for a ViT on ImageNet.
\end{itemize}

\noindent

\section{Related Work}\label{sec:related work}\label{sect:related_work}
\noindent\textbf{Neural Architecture Search.} NAS aims to replace human experts in architecture design. Early approaches used agents like RNNs or evolutionary algorithms to sample and evaluate architectures, which are computationally expensive and impractical for large datasets \cite{zoph2016neural,zhong2018practical,baker2016designing,Tan2018MnasNetPN,Tan2019EfficientNetRM,chen2018reinforced,negrinho2017deeparchitect,Kandasamy2018Neural_neurips,White2021BANANAS_aaai}. More recent studies employ weight-sharing supernets \cite{Cai2018ProxylessNASDN, Liu2018DARTSDA, dong2019searching, brock2017smash}, optimizing both supernet weights and architecture indicators through gradient-based methods \cite{Liu2018DARTSDA,Cai2018ProxylessNASDN, Wu2018FBNetHE}. However, these methods can introduce bias between sub-models and rely heavily on initialization. One-shot approaches ensure fairness among sub-models by sampling and rating them using inherited weights from a supernet \cite{Guo2019SinglePO, Chu2019FairNAS_arxiv, brock2017smash, bender2018understanding}, but there remains a performance gap between shared-weight proxy sub-models and retrained standalone models, as observed in prior work \cite{bender2018understanding, Chu2019FairNAS_arxiv, anonymous2020improving}.
Specifically, \cite{anonymous2020improving} provided an analysis using a Bayesian view \cite{Baxter1997Bayesian_ml} to identify that this gap narrows as the amount of weight-sharing sub-models decreases. However, its assumption that the weights of each network layer are independent might be too idealistic.
DNA (originally proposed in our conference version \cite{li2020block}) addresses architecture rating issues by optimizing the supernet differently.

\noindent\textbf{Knowledge Distillation.} KD is a model compression method transferring knowledge from a trained teacher to an efficient student. Approaches fall into two categories: soft label-based methods (\cite{Ba2013DoDN,Hinton2015DistillingTK}) and feature-based methods (\cite{Romero2014FitNetsHF, Yim2017AGF, Wang2018ProgressiveBK}). \cite{Yim2017AGF} trains a student to mimic the teacher's behavior in multiple hidden layers. \cite{Wang2018ProgressiveBK} uses a progressive block-wise KD scheme to transmit knowledge, easing block-wise KD optimization. Unlike the progressive scheme, our method employs parallelized KD to reduce time consumption and teacher-student gap. While previous KD works optimize network weights, our method focuses on optimizing neural architecture.

\noindent\textbf{Progressive Learning.}
Self-distillation enhances teacher-student model accuracy by reusing a student as a new teacher \cite{furlanello2018born}. It progressively improves regularization and reduces overfitting \cite{mobahi2020self}.
Progressive learning extends beyond KD, like \cite{xie2020self} using pseudo labels to train a new student, improving label quality and model performance. Our work, distinctively, focuses on progressive neural architecture search rather than network weight optimization.

\noindent\textbf{Self-Supervised Learning.}
Recent NAS methods often emphasize the label-free nature of architecture search \cite{Liu2020Are_eccv,Zhang2021Neural_cvpr,Yan2020Does_neurips,Wang2021Joint_tnnls}, considering self-supervision optional. In contrast, DNA++ aims to replace the supervisor with self-supervision due to potential bias. DNA++ significantly differ from BossNAS \cite{Li2021BossNAS_iccv} in NAS techniques, self-supervised learning, and scalability. DNA++ uses a heterogeneous teacher-student model (standard net vs. supernet), while BossNAS employs a homogeneous model (supernet vs. supernet).  To enable the heterogeneous model, DNA++ introduces new SSL techniques to prevent mode collapse and encourage output divergence among samples, effectively combining self-supervision with weight-sharing NAS. Lastly, DNA++ offers diverse application possibilities, such as using a CNN for self-supervision and searching for ViTs, thanks to its heterogeneous model.

\noindent\textbf{NAS for ViTs.} Some prior work focused on NAS for ViTs. The key differences between the DNA family and them are as follows:
GLiT \cite{chen2021glit} ingeniously introduces a novel search space for ViTs and employs an intelligent hierarchical search algorithm. However, it may exclude certain architectures with potential during the initial screening phase, limiting their consideration.
ViTAS \cite{su2022vitas} ingeniously presents a novel cyclic method for supernet training, aiming for fairness \cite{Chu2019FairNAS_arxiv}. While promising, it doesn't address reducing the search space or the challenge of inaccurate architecture ranking, which the DNA family target.
AutoFormer \cite{chen2021autoformer} cleverly employs a supernet training methodology inspired by BigNAS \cite{yu2020bignas}. However, the correlation between subnets' accuracy and their train-from-scratch performance remains unknown, making them somewhat opaque. In contrast, the DNA family resolves the issue of inaccurate architecture ranking in weight-sharing NAS.
NASViT \cite{gong2021nasvit} smartly focuses on a novel method for neural network optimization, which could benefit our "from-scratch" retraining process. However, this paper may not cover the application of this optimization to our retraining process, which we plan to explore in future work.
We conducted comprehensive quantitative comparisons, demonstrating that the DNA family consistently achieve superior accuracy compared to these state-of-the-art NAS methods applied to ViTs \cite{chen2021searching}.

\vspace{-11pt}
\section{Methodology}\label{sect:alg}

We start by analyzing weight-sharing NAS's dilemma (Section \ref{sect:dilema}). 
To tackle this issue, we partition the search space into blocks (Section \ref{sect:dna}). Using these block-wise representations, we develop three NAS models employing neural architecture distillation techniques to address various weight-sharing NAS challenges.
Section \ref{sec:supernetdistill} introduces DNA with supervised learning, offering searchable widths/depths and resource adaptability.
Section \ref{sect:idna} outlines DNA+ with progressive learning.
Section \ref{sect:dna++} presents DNA++ with self-supervised learning.

\vspace{-11pt}
\subsection{Basic Analysis of Weight-sharing NAS's Dilemma}\label{sect:dilema}

Finding optimal architectures requires multiple/greedy trials, i.e., training many/all architectures separately to assess their quality, which is computationally unaffordable for many researchers. Recent works give up training each candidate individually; instead, they encode the search space into a weight-sharing supernet\footnote{A \textbf{supernet} is a general concept widely used in the NAS community \cite{Guo2019SinglePO,Cai2018ProxylessNASDN,Liu2018DARTSDA,Wu2018FBNetHE,Pham2018EfficientNA}, referring to a directed acyclic super-graph covering a whole search space with each node representing the feature maps and each edge representing a connection between the nodes with a particular operation (e.g., a convolution). Each subnet in the supernet represents a candidate architecture in the search space.} and train all architectures in the supernet concurrently. These methods are called weight-sharing NAS, which significantly improves search efficiency. Despite high efficiency, the effectiveness of weight-sharing NAS is unsatisfactory. We analyze the culprit here.

\begin{figure*}
\begin{center}
\includegraphics[width=0.70\linewidth]{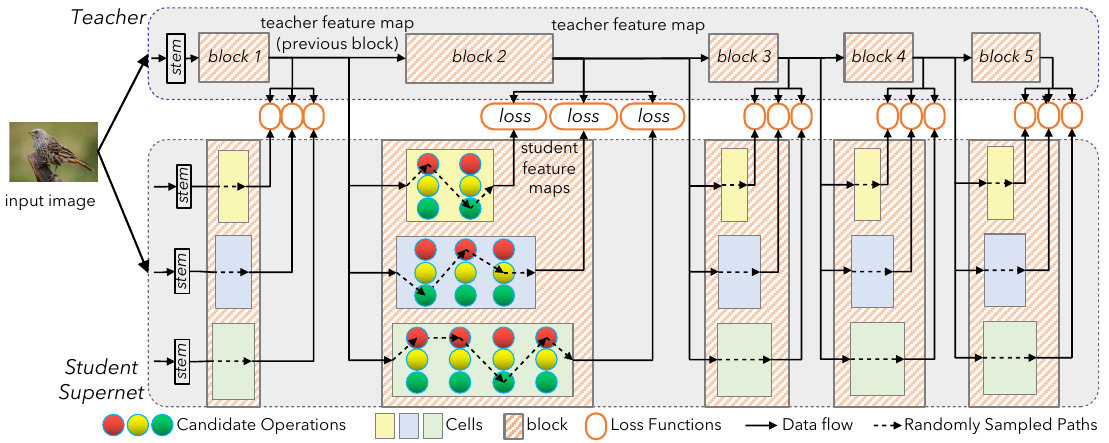}
\end{center}
\vspace{-15pt}
 \caption{\small{Illustration of our DNA. The teacher's preceding feature map is used as input for both teacher and student blocks. Each cell of the supernet is trained independently to mimic the behavior of the corresponding teacher block by minimizing the L2 distance between their output feature maps. The dotted lines indicate randomly sampled paths in a cell.} \textbf{(Best viewed in color)}} \label{fig:training}
 \vspace{-11pt}
\end{figure*}

Formally, let $\alpha_j$ and $\psi_j$ denote an architecture and its network weights, respectively. NAS is a problem to find the optimal pair ($\alpha_j^*$, $\psi^*_j$) such that the model performance is maximized. To perform searching, people need to assess many architectures' quality as rewards for adjusting search policies, i.e., they need to train each architecture separately to obtain $\psi_j^* = \mathop{\arg\min} \mathcal{L}(\psi_j|\alpha_j)$\footnote{$\mathcal{L}$ is the loss; we ignore the input data and the labels for simplicity.} and use $\psi_j^*$ to rate the architecture. However, training each architecture is quite inefficient. Weight-sharing NAS methods formulate the search space $\mathcal{A}$ into an over-parameterized supernet such that each candidate architecture $\alpha_j$ is a subnet of the supernet. They train all architectures concurrently in the supernet to obtain optimal supernet weights $\Psi^{sup} = \mathop{\arg\min}\mathcal{L}(\Psi|\mathcal{A})$. Then, they extract weights from the supernet for each subnet for validation and use this validation accuracy to rate the subnet. The dilemma of weight-sharing NAS is that there is a low correlation between the predicted and the actual architecture quality. To analyze this conjecture, we first have a generalization boundedness theorem.

\noindent
\textcolor{black}{
\begin{tcolorbox}[colframe=black,colback=lightgray!30,boxrule=0.5pt,boxsep=2pt,left=1pt,right=1pt,top=0pt,bottom=0pt]
\begin{theorem}\label{theorem:generalization}
\textbf{(Generalization boundedness).} For any subnet $\alpha_j$, we use $\psi_j^{sup}$ to denote its sub-optimal weights extracted from a trained supernet and use $\psi_j^{*}$ to denote its ideal weights when trained alone. 
Then, the Frobenius norm of $\psi_j^{sup}$ is upper bounded by:\begin{small}\begin{equation}\label{eqn:bound-a}
\big\| \psi_j^{sup} \big\|_F 
 \le \sqrt{\mathcal{C}_1 + \mathcal{C}_2\sum\limits_{t=1}^{T} \big| \mathcal{C}_3 - \mathcal{L}(\psi_t^*| \alpha_t)\big| } ,
\end{equation}\end{small}where \begin{small}$\mathcal{C}_1\ge 0$\end{small}, \begin{small}$\mathcal{C}_2 \ge 0$\end{small}, and \begin{small}$\mathcal{C}_3$\end{small} are constants.
\end{theorem}
\end{tcolorbox}
}

\noindent
\textbf{Remark.} The Proof of Theorem \ref{theorem:generalization} is in the appendix. Theorem \ref{theorem:generalization} shows that using the trained supernet weights, a subnet's model complexity (usually measuring the generalization ability) has an upper bound associated with $T$. Increasing the search space leads to a poorer generalization ability, further implying that the architecture quality estimated by the supernet weights is not predictive of the actual quality. Suppose an architecture has excellent quality, but its generalization ability based on the supernet weights is poor. Then, its validation accuracy might be low, and its quality will be underestimated. \textbf{In summary}, the oversized search space is the devil/culprit because a big search space would result in poor generalization ability and further lead to inaccurate architecture rating, which finally leads to search ineffectiveness. On the contrary, reducing the search space can improve the effectiveness of weight-sharing NAS.

\textcolor{black}{
Theorem \ref{theorem:generalization} underscores the potential for block-wise supervision to enhance architecture ranking, a concept that resonates with ``Modular Learning \cite{ballard1987modular}''. Modular learning effectively demonstrated the advantages of decomposing neural network training into modular components and independently training each model, showcasing a multitude of benefits.
In our context, Theorem \ref{theorem:generalization} supports the idea that modular learning ensures fair and thorough training of each subnet, a key factor contributing to improved ranking.
}

\vspace{-11pt}
\subsection{Modularizing Search Space into Blocks
}\label{sect:dna}

As discussed, to improve weight-sharing NAS's effectiveness, one should reduce the search space. Unfortunately, direct replacement of a large space with a small one is inadvisable since it leads to a small accuracy range, making the search meaningless. Keeping the whole search unchanged, we modularize the large space into blocks. The candidate number in each block is thus significantly smaller than that in the entire search space. That is, we divide the supernet $\mathcal{U}$ into $N$ blocks \emph{w.r.t} the depth of the supernet:
\begin{small}\begin{equation}
\left\{
\begin{aligned}
\mathcal{U} &= \mathcal{U}_N \circ \dots \circ \mathcal{U}_{k+1} \circ \mathcal{U}_{k} \circ\dots \circ \mathcal{U}_1\\
\Psi &= \Big[\Psi_N ; \dots ; \Psi_{k+1} ; \Psi_{k} ; \dots ; \Psi_1\Big]\\
\mathcal{Z} &= \Big[\mathcal{Z}_N,; \dots ; \mathcal{Z}_{k+1};\mathcal{Z}_{k}; \dots ; \mathcal{Z}_1\big\}\Big]\\
\end{aligned} ,
\right.
\end{equation}\end{small}where $\mathcal{U}_{k+1} \circ \mathcal{U}_{k}$ represents that the ($k$+$1$)-th block is connected to the $k$-th block in the supernet. $\Psi_k$ is the network weights of the $k$-th block. $\mathcal{Z}_k = \big\{\mathcal{X}_{k}, \mathcal{Y}_{k}\big\}$ are the input data and the supervision of the $k$-th block. We optimize each block of the supernet separately:\begin{small}\begin{equation}\label{eq:loss-block}
\Psi_k^{sup} = \mathop{\arg\min}_{\Psi_k} {\mathcal{L}_{\text{train}}(\Psi_k | \mathcal{A}_k}, \mathcal{X}_k, \mathcal{Y}_k ), ~~~~i = 1,2\cdots,N,
\end{equation}\end{small}where $\mathcal{A}_k$ denote the search space of the $k$-th block.

\textcolor{black}{\textbf{Blocky \emph{vs.} entire search space}. Let $d_k$ and $c$ denote the depth of the $k$-th block and the number of candidate operations in each layer. Then, the size of the search space of the $k$-th block is $c^{d_k}$; the size of the entire search space is \begin{small}$\prod_{k=0}^N c^{d_k}$\end{small}. This indicates an exponential reduction in search space, i.e., $(\prod_{k=1}^N c^{d_k})/c^{d_k}$. In our experiment, the blocky search space is significantly smaller than the entire search space (e.g., reduction ratio$\approx 1e^{\frac{15}{N}}$), ensuring effective weight-sharing NAS.} Finally, the architecture is searched across the different blocks in the \underline{whole search space $\mathcal{A}$}:\begin{small}\begin{equation}\label{eq:loss-supernet}
\begin{aligned}
&\alpha^{*} = \mathop{\arg\min}\limits_{\forall \alpha \in \mathcal{A}}\sum_{k=1}^{N} \lambda_k \mathcal{L}_{\text{val}}\Big(\Psi_k^{\sup}|\mathcal{A}_k, \mathcal{X}_k, \mathcal{Y}_k\Big) \\
s.t. ~~&\Psi_k^{sup} = \mathop{\arg\min}_{\Psi_k} \mathcal{L}_{\text{train}}(\Psi_k | \mathcal{A}_k, \mathcal{X}_k, \mathcal{Y}_k), k = 1,\cdots,N
\\
\end{aligned},
\end{equation}\end{small}where $\lambda_k$ represents the loss balance.

\noindent
\textcolor{black}{
\textbf{Remark.} \emph{It's crucial to highlight that, although we use block-wise supervision to help train the supernet to provide block-wise local scores, our architecture search is still performed in an entire architecture by taking all the local scores into consideration. Previous works \cite{anonymous2020improving} assume that the weights at each layer in deep nets are independent. In contrast, we consider the weights at different layers of the deep nets to be highly dependent on each other.}
Our algorithm in Section \ref{sec:supernetdistill} allows us to precisely identify the highest-reward architecture among the full search space of $10^{17}$ architectures.
Furthermore, our approach can also search for architectures with varying depths and widths while adhering to specific computational constraints.
}

\vspace{-11pt}
\subsection{DNA: Distillation via Supervising Learning}\label{sec:supernetdistill}

Although we motivate well in Section \ref{sect:dna}, a technical barrier is that we lack internal ground truth in Eqn. \eqref{eq:loss-block}. 
Inspired by knowledge distillation (KD) \cite{Hinton2015DistillingTK}, we use the hidden features of a teacher as the supervision. Let $\mathcal{Y}_k$ be the output tensor of the kth block a teacher and $\hat{\mathcal{Y}}_k(\mathcal{X}_k)$ be the output tensor of the kth block of the supernet. We take the L2 norm as the cost function. The loss function in Eqn. \eqref{eq:loss-block} can be written as:\begin{small}\begin{equation}\label{eq:distill-loss}
\mathcal{L}_{\text{train}}(\Psi_k | \mathcal{A}_k, \mathcal{X}_k, \mathcal{Y}_k) = \frac{1}{K} \left\| \mathcal{Y}_k - \hat{\mathcal{Y}}_k(\mathcal{X}_k) \right\|_2^2,
\end{equation}
\end{small}where $K$ denotes the neuron numbers in $\mathcal{Y}$. Moreover, inspired by the remarkable success of the transformers \cite{Vaswani2017Attention_nips,Devlin2019pretraining_naacl,Dong2021EfficientBERT_emnlp} that discards the inefficient sequential training of RNN, we parallelize our supernet training analogously. Specifically, we use the output of the (k-1)th block of the teacher as the input of the (k-1)th block of the supernet, i.e., we use $\mathcal{Y}_{k-1}$ to replace $\mathcal{X}_k$ in Eqn. \eqref{eq:distill-loss}. Thus, the search can be sped up in parallel. Eqn. \eqref{eq:distill-loss} can be written as:\begin{small}\begin{equation}\label{eq:distill-loss-re}
\mathcal{L}_{\text{train}}\Big(\Psi_k | \mathcal{A}_k, \mathcal{Y}_{k-1}, \mathcal{Y}_k \Big) = \frac{1}{K} \left\| \mathcal{Y}_k - \hat{\mathcal{Y}}_k(\mathcal{Y}_{k-1}) \right\|_2^2.
\end{equation}
\end{small}Fig.\ref{fig:training} shows a pipeline of DNA.

\begin{figure}
\begin{center}
\includegraphics[width=1.0\linewidth]{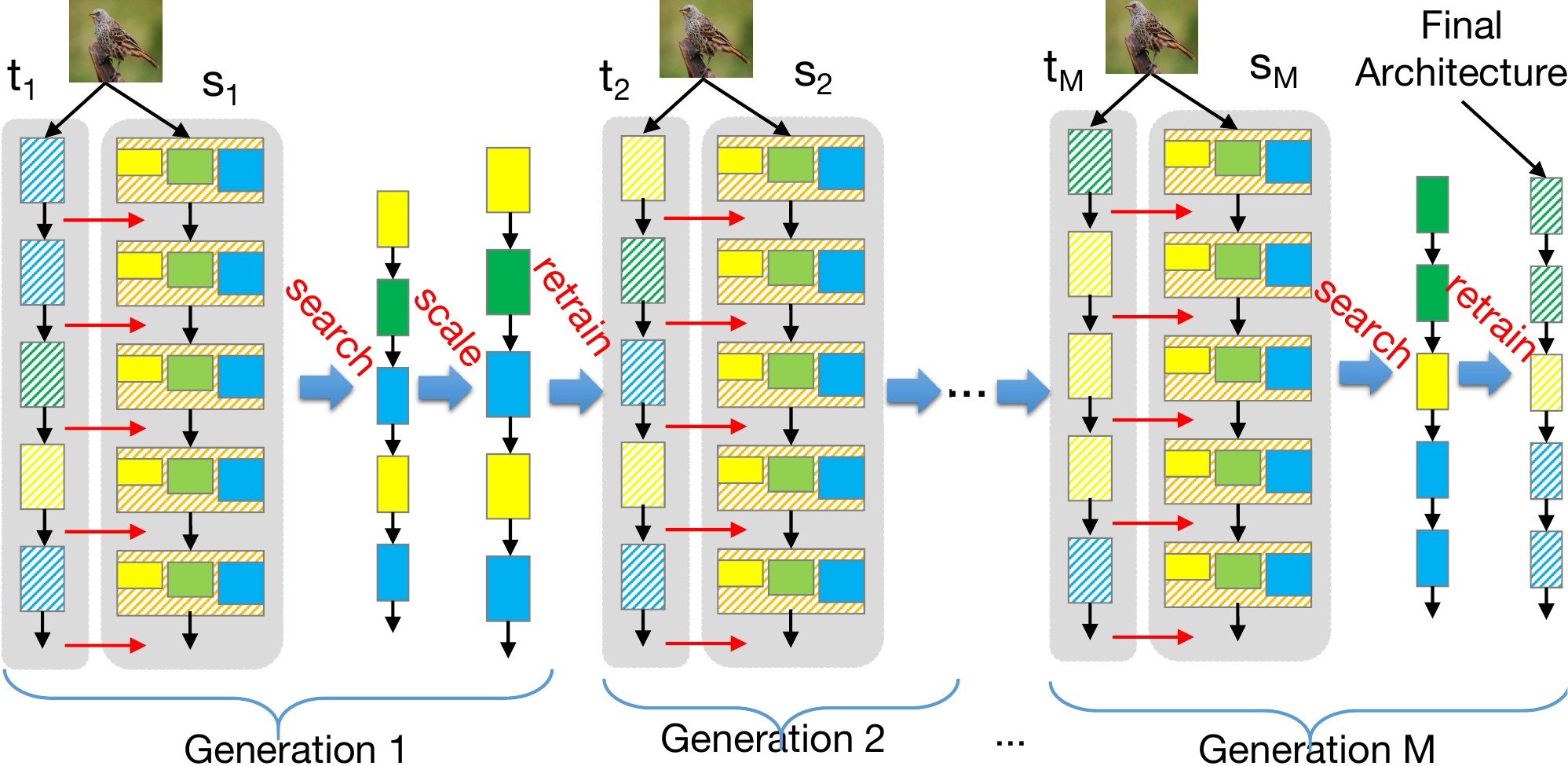}
\end{center}
\vspace{-15pt}
 \caption{\small{Illustration of DNA+. In the first generation, we use an existing model as the teacher model. Then, at each consecutive generation, a new teacher is obtained by scaling the searched architecture of the previous generation and retraining the scaled architecture. The finally searched architecture is the optimal student $\alpha^{M*}$ in the last generation, which is retrained without scaling.}}\label{fig:born-again}
 \vspace{-15pt}
\end{figure}

\textbf{Searchable depths \& widths.} Automatically allocating each block's model complexity under a particular constraint is vital in NAS. To better imitate the teacher, each block's model complexity should be allocated according to the corresponding teacher block's learning difficulty adaptively. With the input image size and the stride of each block fixed, generally, the computation allocation is only related to the width and depth of each block, which are burdensome to search in weight-sharing NAS because the width and depth are usually pre-defined when designing the supernet. Most previous works include \emph{identity} as a candidate operation for depth search (\cite{bender2018understanding, Liu2018DARTSDA, Cai2018ProxylessNASDN, Wu2018FBNetHE, anonymous2020improving}). However, as pointed out by \cite{Chu2019ScarletNASBT}, adding identity as a candidate operation can lead to supernet convergence difficulty and an unfair comparison of sub-models. Also, adding identity as a candidate operation may lead to redundant search space. For example, a sequence \{conv, identity, conv\} is equivalent to \{conv, conv, identity\}. According to Theorem \ref{theorem:generalization}, this redundant search space can cause search ineffectiveness. Besides, \cite{anonymous2020computation} searched for the layer number with fixed operations for the first step and subsequently searched for three operations with a fixed layer number. However, the operation choice depends on each block's depth, leading to a cumulative search departure, especially when search space increases. Thanks to block-wise supervision, we can train several \emph{cells} with different channel numbers/layer numbers \emph{in parallel} in each block to ensure depth/width variability. Specifically, in each training step, the teacher's previous feature map is fed to three cells with different depths/withs in parallel (e.g., solid line of data flow in Fig. \ref{fig:training}). For each layer of each cell, an operation is randomly sampled from the candidate operations (e.g., dotted line of data flow in Fig. \ref{fig:training}).

\textbf{Constrained search algorithm.} Our typical supernet contains about $10^{17}$ sub-models, stopping us from rating all of them. Previous weight-sharing NAS used random sampling, evolutionary algorithms, and reinforcement learning to sample sub-models from the trained supernet for further rating. In the latest works (\cite{anonymous2020computation,anonymous2020improving}), a greedy search algorithm was used to progressively shrink the search space by layer-wisely selecting the top-performing partial models. In contrast, although using block-wise supervision to calculate blocky local scores, our architecture search is still performed in the global search space by considering all the local scores. We can subtly traverse all the subnets to select the top-performing ones under certain constraints.
\begin{itemize}
    \item \emph{Rating step.} Block-wise supervision enables rating all candidates in a search space. 
    We first compute local scores using block-wise supervision, which is affordable because there are only $10^4$ sub-models in each cell.
    For further acceleration, we process batch data \emph{node by node} in a manner similar to a depth-first algorithm, with the intermediate output of each node saved and reused by subsequent nodes to avoid recalculating it from the root node. 
    The feature sharing evaluation algorithm is outlined in Alg. \ref{alg:eval} in the appendix. By evaluating all cells in a block of the supernet, we can get the evaluation loss of all possible paths in one block. We can quickly sort this list with about $10^4$ elements in a few seconds with a single CPU.
    Each block has such a local-score list\footnote{One may say that we can select the top-1 partial model in each local-score list to assemble the best student. But this short-sighted selection could lead to local minima. In contrast, our architecture search is performed in the global search space by considering all local scores and their computation complexity. Please refer to the ``search step.''}.
    
    \item \emph{Searching step.} Given a computational cost constraint, we should automatically allocate costs to each block using a fair rating metric for different blocks. As an MSE loss is affected by the variance of a teacher's feature map, we use a fairer metric, i.e., relative L1 loss:\begin{small}\begin{equation}\label{eqn:local_score}
\mathcal{L}_{\text{val}}(\Psi_k | \mathcal{A}_k, \mathcal{Y}_{k-1}, \mathcal{Y}_k ) =\frac{|| \mathcal{Y}_k - \hat{\mathcal{Y}}_k(\mathcal{Y}_{k-1}) ||_1}{{K\sqrt{D(\mathcal{Y}_k)}}},
\end{equation}\end{small}where \begin{small}$D(\cdot)$\end{small} measures the variance. The blocky local scores \begin{small}$\mathcal{L}_{\text{val}}\big(\Psi_k | \mathcal{A}_k, \mathcal{Y}_{k-1}, \mathcal{Y}_k\big)$\end{small} are summed for global search. Efficiently, we don't need to compute the cost (e.g., FLOPs) and losses for all $10^{17}$ candidates. With the scores in a local-score list ranked, we propose an efficient search algorithm for visiting all possible models (Alg. \ref{alg:search} in the appendix). First, during a search loop, if the cost already exceeds the constraint, we use the statement ``continue'' to jump to the next loop iteration. Second, when a model satisfying the constraint is found, we return to the previous block because this model is optimal in the current block.
Third, we get the cost of each candidate operation by a pre-calculated lookup table to save time.
\textcolor{black}{The approach can be likened to a hiking expedition from a starting point A to a destination B, with numerous possible routes to explore. Each path yields a different reward. If we identify point C as a key milestone within the optimal path from A to B and, furthermore, discover the highest-reward path from A to C, denoted as $\overline{AC}$, then it logically follows that the best path from A to B via C must include $\overline{AC}$. Any path that does not incorporate $\overline{AC}$ is not optimal. Leveraging this concept, we are able to efficiently and accurately pinpoint the highest-reward architecture from the extensive pool of $10^{17}$ samples.}
\end{itemize}

\begin{figure*}
\begin{center}
\includegraphics[width=0.6\linewidth]{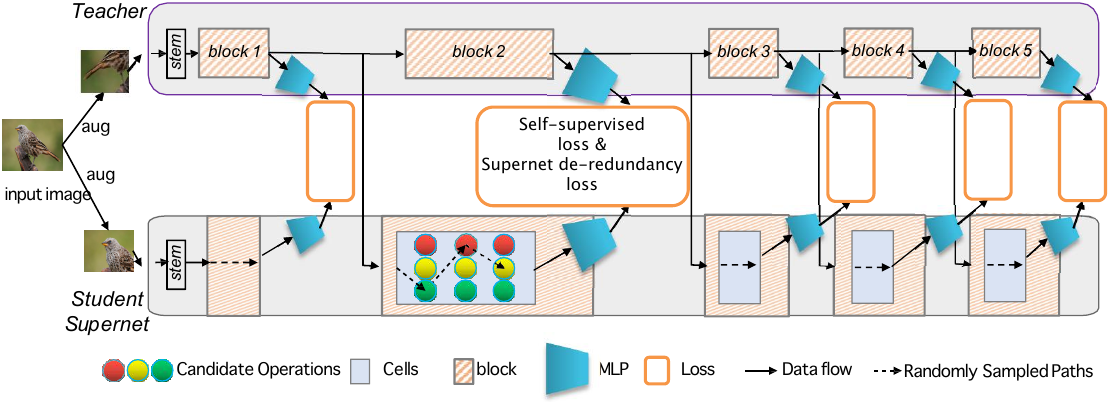}
\end{center}
\vspace{-15pt}
 \caption{\small{Illustration of DNA++. It uses self-supervisions to replace existing supervising teachers, avoiding architecture shifts (referring to a phenomenon that students with a similar architecture to a teacher tend to be favored when the teacher is traditional). DNA++ contains two losses, i.e., a self-supervised loss and a particular loss to remove redundant non-learnable supernets.}}\label{fig:selfsup}
 \vspace{-15pt}
\end{figure*}

\vspace{-11pt}
\subsection{DNA+: Distillation via Progressive Learning}\label{sect:idna}

After searching, our searched architecture has super good performance. If we scale our search architecture to the same size as the teacher, our scaled architecture can significantly outperform the original teacher. This impressive result encourages us to perform a cascade NAS. We iteratively scale up our searched architecture to be a new teacher and search for a new architecture guided by the new teacher generation by generation, mimicking the merit of a born-again neural network \cite{furlanello2018born} with multiple generations of knowledge transfer. This forms another DNA version, i.e., DNA+.

In the first generation, we use an existing model (e.g., EfficientNet-B1) as the teacher model, i.e.:\begin{equation}\label{eqn:1_g}
\mathcal{Y}_k^1 = f(\mathcal{X}_k^1 | \beta_k^1, \Psi_k^1),
\end{equation}where $\beta^1$ and $\Psi^1$ are the architecture and network weights of the first generation teacher. $\mathcal{X}^1$ and $\mathcal{Y}^1$ are the input and output. Here, we use $k$ to denote the kth block. Substituting $\mathcal{Y}_k^1 $ into Eqn. \eqref{eq:loss-supernet}, we can train the first generation supernet and search for an optimal architecture $\alpha^{1*}$ of Generation 1.

At each consecutive generation, a new teacher is obtained by scaling the searched architecture of the earlier generation using the scaling strategy from \cite{Tan2019EfficientNetRM} (denoted as $\mathcal{S}caling$ in Eqn. \eqref{eqn:other_g}). Then, we retrained the new teacher model from scratch, which is used as the block-wise supervision in this generation:\begin{equation}\label{eqn:other_g}
\left\{
\begin{aligned}
&\beta^m = \mathcal{S}caling (\alpha^{(m-1)*}) \\
&\mathcal{Y}_k^m  = f(\mathcal{X}_k^m | \beta_k^m, \Psi_k^m)
\end{aligned}.
\right.
\end{equation}The search process in each consecutive generation is similar to the first generation. We perform the search iteratively for several generations to saturate. At the end of the procedure, the finally searched architecture is the optimal student $\alpha^{M*}$ in the last generation, which will be retrained without scaling. Fig. \ref{fig:born-again} shows a pipeline of DNA+.

\textcolor{black}{\textbf{Kind reminder:} We would like to highlight the fact that the improvement in architectural ranking is primarily attributed to block-wisely supervised learning. Knowledge distillation just serves as a means to achieve this block-wisely supervised learning.}

\vspace{-11pt}
\subsection{DNA++: Distillation via Self-Supervised Learning}\label{sect:dna++}

Despite the high search performance of DNA and DNA+, their searched results might be influenced by the teacher architecture. Specifically, according to Eqn. \eqref{eqn:local_score}, candidates with operations more similar to the teacher tend to have a better local score in each block, resulting in a biased architecture ranking. When the search space is close to the teacher network, this bias is not apparent. \textcolor{black}{But as suggested by \cite{Li2021BossNAS_iccv}, if the search space is more divergent, this bias will be amplified. Fortunately, without access to an existing teacher, self-supervised NAS methods have been proven capable of achieving comparable performance to supervised NAS methods. Hence, we use self-supervision as an alternative to the teachers' supervision to reduce the ranking bias caused by the use of existing teacher models.}

The most effective strategies to use self-supervision in computer vision include contrastive learning \cite{Chen2020Simple_icml,Dong2021EfficientBERT_emnlp,Grill2020Bootstrap_neurips}, metric learning \cite{Wang2021Solving_iccv}, and generation \cite{wang2022semantic_cvpr}. In these strategies, people make two independent data augmentations for each image and input them into a learnable and non-learnable network. Representations are learned by fully tapping the similarity of the two outputs. Note that the non-learnable network is usually a moving average of the learnable network (i.e., the mean teacher) \cite{Chen2020Simple_icml,Dong2021EfficientBERT_emnlp,Grill2020Bootstrap_neurips,Wang2021Solving_iccv} or a copy of the learnable network with its gradients being stopped \cite{Chen2021Exploring_cvpr}. In other words, the learnable and non-learnable networks need to be identical; otherwise, learning will collapse. This makes it very difficult to introduce contrast learning and metric learning into NAS. \textcolor{black}{To solve this problem, \cite{Li2021BossNAS_iccv} and \cite{Peng2021Pi_iccv} set both the teacher and the learning network as supernets, which makes them very bloated and inelegant. Fortunately, inspired by \cite{Bardes2021VICReg_arxiv}, we are able to elegantly solve this problem by introducing DNA++.}

\textbf{Overview \& Notation}. An overview of DNA++ is shown in Figure 2. For each block, we first use two MLPs to extract its features. The features are denoted as $\mathcal{Z}_k^j$ and $\hat{\mathcal{Z}}_k^j$, where $j$ denotes the j-th sample in a batch and $k$ denotes the k-th block. Then we impose two losses on the features, namely a self-supervised loss and a particular loss to remove redundant non-learnable supernets.

\textbf{Self-Supervised Loss.} A basic loss in self-supervised learning is the distance metric loss \cite{Grill2020Bootstrap_neurips,Chen2020Simple_icml,Wang2021Solving_iccv}, which aims to learn the similarities and dissimilarities between features. To stabilize optimization and ease self-supervised learning, regulation is usually also required. Therefore, our self-supervised loss is defined as:\begin{small}\begin{equation}\label{eqn:ssl}
\begin{aligned}
\mathcal{L}_\text{SSL} = &\frac{\lambda_1}{KM}\sum\limits_{k=1}^K\sum\limits_{j=1}^M\| \mathcal{Z}_k^j - \hat{\mathcal{Z}_k^j}\|_2^2 +  \\
& \frac{\lambda_2}{KC}\sum\limits_{k-1}^{K}(\|\text{Off-diag}(Cov(\mathcal{Z}_k))\|_2^2  +\|\text{Off-diag}(Cov(\hat{\mathcal{Z}_k}))\|_2^2), \\
\end{aligned}
\end{equation}\end{small}where the first term is a distance metric loss, and the second term is a regularization loss. The role of the regularization term is seen in \cite{Zbontar2021Barlow_icml}. Here $K$ is the block number, $M$ is the batch size, and $C$ is the channel number. $\lambda_1$ and $\lambda_2$ are used for loss balance. $\text{Off-diag}$ represents taking the off-diagonal elements of a matrix. $Cov$ stands for calculating the covariance matrix of a vector set, i.e., $Cov(\mathcal{Z}_k) = \frac{1}{M-1}\sum_{j=1}^{M}(\mathcal{Z}_k^j - \bar{\mathcal{Z}_k})(\mathcal{Z}_k^j - \bar{\mathcal{Z}_k})^T$, where $\bar{\mathcal{Z}_k} = \frac{1}{M}\sum_{j=1}^{M}\mathcal{Z}_k^j$.

Eqn. \eqref{eqn:ssl} is a standard self-supervised loss, and it inherits one of its inconveniences, i.e., it requires a redundant non-learnable network for optimization. In the context of weight-sharing NAS, if only Eqn. \eqref{eqn:ssl} is used, a redundant non-learnable supernet is inevitably required, making optimization difficult and inelegant. Next, we aim to remove this redundant non-learnable supernet.

\textbf{Avoiding Redundant Supernets.} The reason why the existing self-supervised algorithms rely on redundant non-learnable networks is that if there are only learnable networks, the optimization will collapse into a shortcut, such as the output of different samples being all the same constant vector. Hence, one possible way to avoid the redundant non-learnable supernet is to encourage the outputs of different samples to be divergent. Inspired by \cite{Bardes2021VICReg_arxiv}, we impose a loss function on the output of different samples, maximizing their variance, i.e.:\begin{equation}\label{eqn:var}
\mathcal{L}_{SDR} = \frac{\lambda_3}{CK}\sum\limits_{k=1}^K\sum\limits_{c=1}^{C}\max(0, \gamma - std(\mathcal{Z}_{k,c})),
\end{equation}where $\max(0, )$ is a hinge-like loss, which is widely used in SVMs. $\gamma$ is a threshold. $std$ represents computing the standard deviation, i.e., $std(\mathcal{Z}_{k,c}) = \sqrt{\text{Var}(\mathcal{Z}_{k,c}) + \epsilon}$. $\epsilon$ is a small constant, and $\mathcal{Z}_{k,c}$ denotes the c-th channel neuron of the k-th block. After adding the above supernet de-redundancy loss $\mathcal{L}_{SDR}$, we no longer need an extra non-learnable supernet to train our supernet. Our final loss is $\mathcal{L}_{SSL} + \mathcal{L}_{SDR} $. In this way, self-supervision is cleverly combined with weight-sharing NAS to form our DNA++.

\emph{In DNA++, our self-supervision can be a CNN, and the student supernet is a ViT. Thus, we can search for ViTs.} Fig. \ref{fig:selfsup} shows a pipeline of DNA++.

  \vspace{-11pt}
\section{Experiments} \label{sect:exp}

\subsection{Datasets.}\label{sect:dataset}

We evaluated our method on ImageNet \cite{deng2009imagenet}, a benchmark for practical NAS methods. For architecture search, we created a 50k-image validation set by randomly selecting 50 images from each class of the original training set. The rest of the images were used for supernet training. After search, both the found architectures and scaled architectures were \textbf{retrained from scratch} on the original training set \textbf{without} teacher network supervision and tested on the original validation set. Additionally, we assessed transferability on CIFAR-10 and CIFAR-100 \cite{Krizhevsky09cifar}. For a more comprehensive evaluation of generality, we conducted semantic segmentation tasks on PASCAL VOC 2012 \cite{Everingham2015Pascal_ijcv} and ADE20K \cite{zhou2019semantic}, and object detection tasks on MS-COCO dataset \cite{Lin2014Microsoft_eccv}.

\textbf{ImageNet NAS Bench.} To evaluate NAS methods, prior approaches commonly retrain searched architectures from scratch, making it hard to discern whether improvements are due to NAS effectiveness or retraining techniques. NAS benchmarks exist but are often based on small datasets (e.g., CIFAR-10 or ImageNet-Tiny) and cell-based search spaces. Addressing this, we created a benchmark on full ImageNet using the mobile-setting search space, comprising 23 randomly sampled architectures with top-1 accuracies ranging from 73.58\% to 75.52\% that align with the true accuracy distribution. More details are in Fig. \ref{fig:bench} in the appendix.

\vspace{-11pt}
\subsection{Search Spaces and Architecture Details}\label{sec:search_spaces}

Our search spaces are defined by selecting operations from operation candidates and selecting cells with varying channel and layer numbers. They include ViT and MBConv search spaces.

\textbf{MBConv Search Spaces.} In these spaces, operation candidates are MBConvs
\cite{Sandler2018MobileNetV2IR}. We use two MBConv search spaces. The first one used in Section \ref{sec:performance} is similar to most of the recent NAS works \cite{Tan2018MnasNetPN, Tan2019EfficientNetRM, Chu2019ScarletNASBT, Chu2019MoGASB} to ensure a fair comparison, which has 6 operation types in total, i.e., a combination of convolution kernel sizes of \{3, 5, 7\} and expansion rates \{3, 6\}. Table \ref{tab:searchspacecmp} in the appendix provides a detailed search space comparison with existing NAS methods. For fast evaluation in Section \ref{sec:effctiveness} and \ref{sec:ablation}, a smaller search space with 4 operation types (kernel sizes of \{3, 5\} and expansion rates \{3, 6\}) is used. Besides, we also search for cells with different channel numbers and layer numbers, as introduced in Section \ref{sec:supernetdistill}. There are 3 cells in each of the first five blocks and 1 cell in the last block. The layer numbers and channel numbers of each cell are shown in Table \ref{tab:design}. The whole search space contains $2\times10^{17}$ architectures.

We use EfficientNet B7 \cite{Tan2019EfficientNetRM} as our teacher for supernet training due to its state-of-the-art performance and relatively low computational cost compared to ResNeXt-101 \cite{Xie2016AggregatedRT} and other manually designed models. We divide the teacher model into 6 blocks in sequence according to their filter numbers. The details of these blocks are shown in Table \ref{tab:design}.

\begin{table}
\centering
\caption{\small{Supernet design. ``$l\#$'' and ``ch$\#$'' represent the numbers of the layers and channels of each cell, respectively.}}\label{tab:design}
\vspace{-10pt}
\resizebox{0.4\textwidth}{!}{
\begin{tabular}[t]{p{1.2cm}|p{0.3cm}p{0.5cm}|p{0.3cm}p{0.5cm}|p{0.3cm}p{0.5cm}|p{0.3cm}p{0.5cm}}
 \toprule
 model&\multicolumn{2}{c|}{teacher}&\multicolumn{6}{c}{student supernet}\cr\hline
 - & \multicolumn{2}{c|}{-} & \multicolumn{2}{c|}{cell 1} & \multicolumn{2}{c|}{cell 2} & \multicolumn{2}{c}{cell 3}\\\hline
 - & $l\#$ & ch$\#$ & $l\#$ & ch$\#$ & $l\#$ & ch$\#$ & $l\#$ & ch$\#$ \\
 \hline
 block 1 & 7 & 48 & 2 & 24 & 3 & 24 & 2 & 32\\
 block 2 & 7 & 80 & 2 & 40 & 3 & 40 & 4 & 40\\
 block 3 & 10 & 160 & 2 & 80 & 3 & 80 & 4 & 80\\
 block 4 & 10 & 224 & 3 & 112 & 4 & 112 & 4 & 96\\
 block 5 & 13 & 384 & 4 & 192 & 5 & 192 & 5 & 160\\
 block 6 & 4 & 640 & 1 & 320 & - & - & - & -\\
 \bottomrule
\end{tabular}
}
\vspace{-20pt}
\end{table}

\textbf{ViT Search Space.} In this space, we build our supernet based on the well-known DeiT \cite{touvron2020deit}. The search space design is inspired by slimmable networks~\cite{yu2018slimmable,li2021dynamic,jiang2023dynamic}, especially DS-ViT \cite{Li2022Ds_tpami}. We define the candidate operations by the embedding dimension (224, 448, 32), the head number (7, 14, 1), the MLP ratio (2.5, 4, 0.5), and the network depth (12, 16, 1). Here, ($a$, $b$, $l$) represents a search space ranging from $a$ to $b$ with interval $l$. The whole search space is divided into 4 blocks, and contains about $1.7\times10^7$ architectures in total.

Although our search space/supernet is ViT-based, we do not use a ViT as a teacher to guide our supernet training. Instead, we adopt EfficientNet-B1 as our teacher due to two reasons. First, although a ViT has high performance, its computational complexity (e.g., in terms of parameter number) is higher than an EfficientNet-B1. Second, the architectural differences between a teacher and a student supernet help to test the effectiveness of our proposed block-wise self-supervision technique in addressing architecture shifts.

\begin{table}
\centering
\caption{\small{Comparison of state-of-the-art NAS models on ImageNet. The input size is $224\times 224$. RA: RandAugment \cite{cubuk2019randaugment}.}}\label{tab:imagenet}
\vspace{-6pt}
\resizebox{0.43\textwidth}{!}{
 \begin{tabular}{p{3.5cm}|p{0.9cm}|p{0.9cm}|p{0.9cm}|p{0.9cm}}
 model &Params & FLOPs & Acc@1 & Acc@5\\
 \hline\hline
 SPOS \cite{Guo2019SinglePO} & - & 319M & 74.3\% & -\\
 ProxylessNAS \cite{Cai2018ProxylessNASDN} & 7.1M & 465M & 75.1\% & 92.5\%\\
 FBNet-C \cite{Wu2018FBNetHE} & - & 375M & 74.9\% & -\\
 MobileNetV3 \cite{Howard2019SearchingFM} & 5.3M & 219M & 75.2\% & -\\
 MnasNet-A3 \cite{Tan2018MnasNetPN} & 5.2M & 403M & 76.7\% & 93.3\%\\
 FairNAS-A \cite{Chu2019FairNAS_arxiv} & 4.6M & 388M & 75.3\% & 92.4\%\\
 MoGA-A \cite{Chu2019MoGASB} & 5.1M & 304M & 75.9\% & 92.8\%\\
 SCARLET-A \cite{Chu2019ScarletNASBT} & 6.7M & 365M & 76.9\% & 93.4\%\\
 PC-NAS-S \cite{anonymous2020improving} & 5.1M & - & 76.8\% & -\\
 MixNet-M \cite{Tan2019MixConvMD} & 5.0M & 360M & 77.0\% & 93.3\%\\
 EfficientNet-B0 \cite{Tan2019EfficientNetRM} & 5.3M & 399M & 76.3\% & 93.2\%\\
 \textcolor{black}{FBNetV3-A0 \cite{dai2021fbnetv3}} & \textcolor{black}{6.3M} & - & \textcolor{black}{78.4\%} & - \\
 \textcolor{black}{FBNetV3-A \cite{dai2021fbnetv3}} & \textcolor{black}{8.6M} & - & \textcolor{black}{79.1\%} & - \\
 \textcolor{black}{AttentiveNAS-A0 \cite{wang2021attentivenas}} & \textcolor{black}{9.1M} & - & \textcolor{black}{77.3\%} & - \\
 \textcolor{black}{AttentiveNAS-A1 \cite{wang2021attentivenas}} & \textcolor{black}{9.6M} & - & \textcolor{black}{78.4\%} & - \\
 \textcolor{black}{AttentiveNAS-A2 \cite{wang2021attentivenas}} & \textcolor{black}{11.3M} & - & \textcolor{black}{78.8\%} & - \\
 \textcolor{black}{AlphaNet-A0 \cite{wang2021alphanet}} & \textcolor{black}{9.1M} & - & \textcolor{black}{77.8\%} & - \\
 \textcolor{black}{AlphaNet-A1 \cite{wang2021alphanet}} & \textcolor{black}{9.6M} & - & \textcolor{black}{78.9\%} & - \\
 \textcolor{black}{AlphaNet-A2 \cite{wang2021alphanet}} & \textcolor{black}{11.3M} & - & \textcolor{black}{79.1\%} & - \\
 \textcolor{black}{EfficientNetV2-B0 \cite{tan2021efficientnetv2}} & \textcolor{black}{7.4M} & - & \textcolor{black}{78.7\%} & - \\
 \textcolor{black}{OnceForAll-1080ti-27 \cite{cai2019once}} & \textcolor{black}{6.5M} & - & \textcolor{black}{76.4\%} & - \\
 \textcolor{black}{OnceForAll-1080ti-22 \cite{cai2019once}} & \textcolor{black}{5.2M} & - & \textcolor{black}{75.3\%} & - \\
 \textcolor{black}{OnceForAll-1080ti-15 \cite{cai2019once}} & \textcolor{black}{6.0M} & - & \textcolor{black}{73.8\%} & - \\
 \textcolor{black}{OnceForAll-1080ti-12 \cite{cai2019once}} & \textcolor{black}{5.9M} & - & \textcolor{black}{72.6\%} & - \\
 \textcolor{black}{OnceForAll-v100-11 \cite{cai2019once}} & \textcolor{black}{6.2M} & - & \textcolor{black}{76.1\%} & - \\
 \textcolor{black}{OnceForAll-v100-09 \cite{cai2019once}} & \textcolor{black}{5.2M} & - & \textcolor{black}{75.3\%} & - \\
 \textcolor{black}{OnceForAll-v100-06 \cite{cai2019once}} & \textcolor{black}{4.9M} & - & \textcolor{black}{73.0\%} & - \\
 \textcolor{black}{OnceForAll-v100-05 \cite{cai2019once}} & \textcolor{black}{5.2M} & - & \textcolor{black}{71.6\%} & - \\
 \hline
 random & 5.4M & 399M & 75.7\% & 93.1\%\\
 \hline
 DNA-a (ours) & 4.2M & 348M & 77.1\% & 93.3\% \\
 DNA-b (ours) & 4.9M & 406M & 77.5\% & 93.3\% \\
 DNA-c (ours) & 5.3M & 466M & 77.8\% & 93.7\% \\
 DNA-d (ours) & 6.4M & 611M & 78.4\% & 94.0\% \\
 \textcolor{black}{\textbf{DNA-d w/ KD (ours)}} & \textcolor{black}{\textbf{6.4M}} & \textcolor{black}{\textbf{611M}} & \textcolor{black}{\textbf{79.72\%}} & \textcolor{black}{\textbf{94.34\%}} \\
 DNA-c w/ AA (ours) & 5.3M & 466M & 77.9\% & 93.9\% \\
DNA-c w/ RA (ours) & 5.3M & 466M & 78.1\% & 94.0\% \\ DNA-d w/ RA (ours) & 6.4M & 611M & 78.9\% & 94.2\% \\
 DNA+ -c w/ AA (ours) & 5.3M & 476M & 78.0\% &  93.9\% \\
 DNA+ -c w/ RA (ours) & 5.3M & 476M & 78.3\% &  94.1\% \\
 \end{tabular}
 }
\vspace{-11pt}
\end{table}

\begin{figure*}[t]
\centering
\subfigure{
\centering
\includegraphics[width=0.31\linewidth]{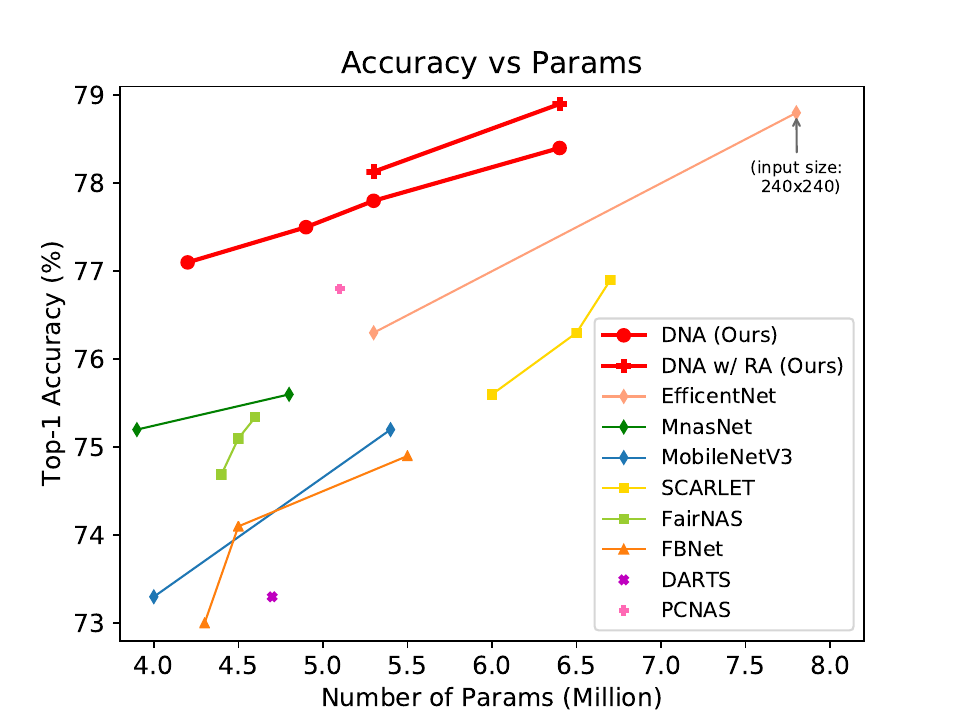}
}
\subfigure{
\centering
\includegraphics[width=0.31\linewidth]{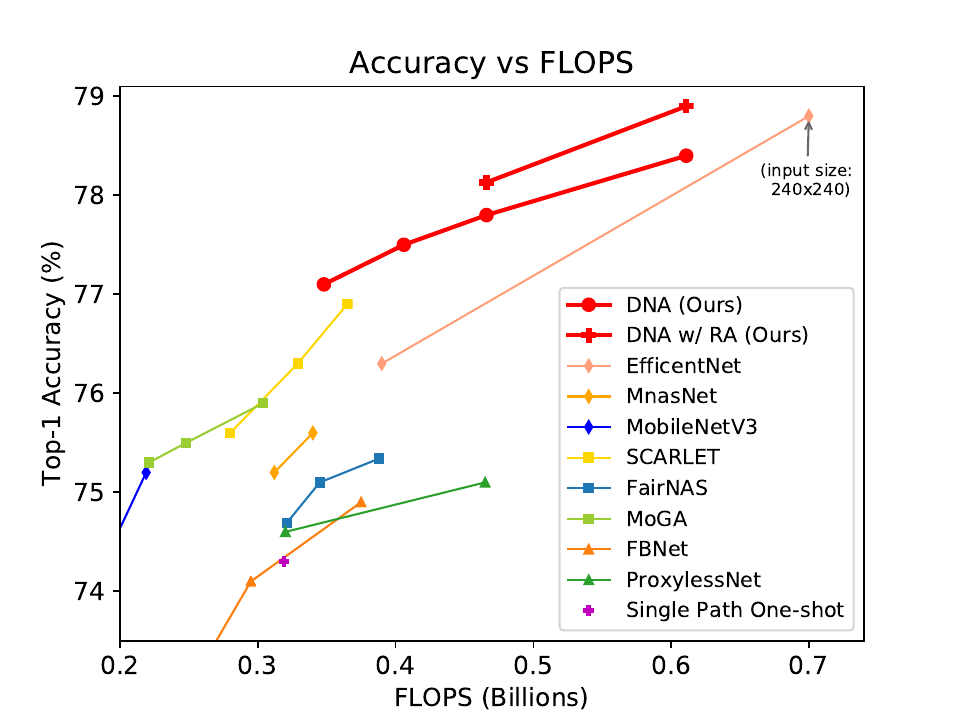}
}
\subfigure{
\centering
\includegraphics[width=0.31\linewidth]{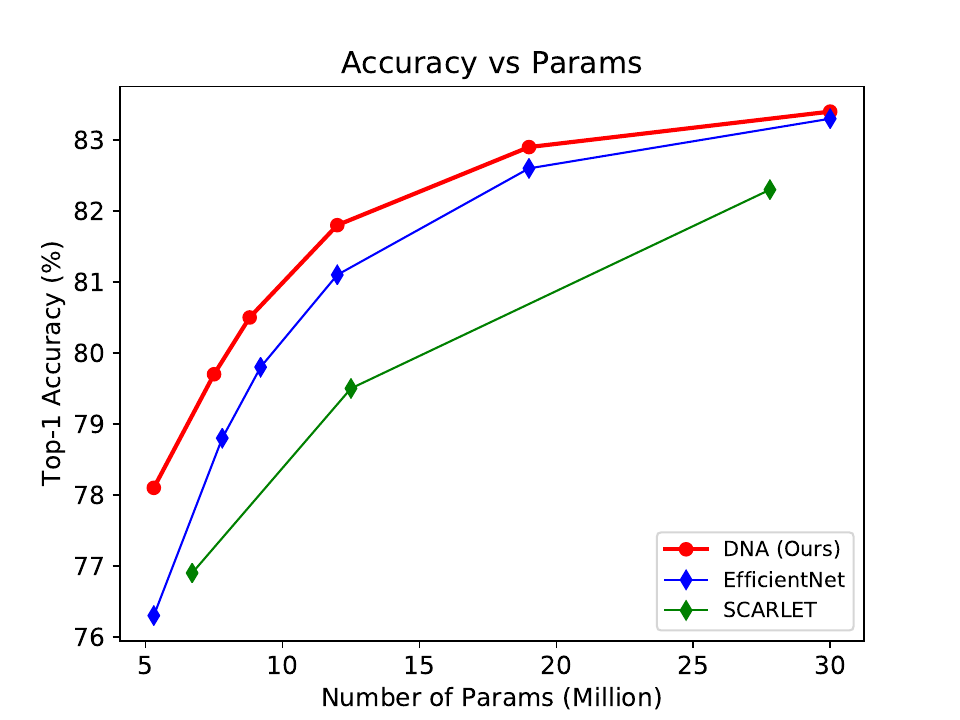}
}
\vspace{-13pt}
\caption{\small{Trade-off between model accuracy and model complexity on ImageNet.
\textbf{Left:}
comparison among our scaled DNA models, EfficientNets, and SCARLET on ImageNet by accuracy \emph{vs.} parameter numbers.
\textbf{Mid:} model accuracy \emph{vs.} parameter numbers; 
\textbf{Right:} model accuracy \emph{vs.} FLOPs. 
}
}\label{fig:paramsota}
\vspace{-15pt}
\end{figure*}

\vspace{-11pt}
\subsection{Searching on MBConv Search Space} \label{sec:performance}

\textbf{Why do we choose DNA and DNA+?} This section only tests DNA and DNA+ because MBConv is efficient and is suitable for algorithms with high training complexity, such as DNA+. We leave DNA++ to be examined in the ViT search space because DNA+ can solve the architecture shift problem in that case.

We train each \emph{cell} in the supernet for 20 epochs \footnote{\scriptsize\textcolor{black}{
In the supernet training stage, we decouple each block, significantly increasing subnet sampling efficiency compared to prior weight-sharing NAS methods. This efficiency allows us to conduct effective training with 20 epochs for each subnet.
\textcolor{black}{In the searching stage, we utilize block-wise supervision and provide block-wise local scores. However, our architecture search encompasses an entire architecture, considering all local scores and avoiding the limitation of sampling only a few subnets.
}
}} individually under the guidance of the teacher's feature map in the corresponding block. For DNA and DNA+, we use 0.002 as the start learning rate for the first block and 0.005 for the rest blocks. We use Adam as our optimizer and reduce the learning rate by 0.9 every epoch.

For DNA/DNA+, it takes three days to train a typical supernet using 8 NVIDIA GTX 2080Ti GPUs. However, simplifying the supernet to contain only one cell in each block takes only one day. With Algorithm \ref{alg:eval}, our cost of architecture rating is about 0.6 GPU days. To search for the best model under certain constraints, we perform Algorithm \ref{alg:search} on CPUs, and the cost is less than one hour.

To retrain our searched (and scaled) architectures on ImageNet, we use a similar setting with \cite{Tan2019EfficientNetRM}, i.e., a batch size of 4,096, an RMSprop optimizer with a momentum of 0.9, and an initial learning rate of 0.256 that decays by 0.97 for every 2.4 epochs. Our models are first trained without additional tricks for a fair comparison, then with RandAugment\cite{cubuk2019randaugment} to stimulate their potential. For transfer learning experiments, we follow the same fine-tuning settings as \cite{Kornblith2018DoBI} and \cite{Huang2019GPipeET}.

\subsubsection{DNA}

\textbf{Top-1 accuracy.} As shown in Table \ref{tab:imagenet}, our DNA models achieve state-of-the-art results compared with the most recent NAS models. \textbf{First}, when searched under a constraint of 350M FLOPs, DNA-a surpasses SCARLET-A with 1.8M fewer parameters. \textbf{Second}, for a fair comparison with EfficientNet-B0, we further perform our DNA under the constraints of 350M FLOPs and 5.3M parameters separately. We thus obtain DNA-b and DNA-c, respectively. Both DNA-b and DNA-c outperform EfficientNet-B0 by a large margin (1.1\% and 1.5\%). Even compared with the recent top-performing MixNet-M, which uses more efficient MixConv operations that we do not use, our DNA-b with fewer parameters beats MixNet-M by 0.4\%. \textbf{Third}, when searched under no constraint, our DNA-d achieves 78.4\% top-1 accuracy with 6.4M parameters and 611M FLOPs. When tested with the same input size ($240\times240$) as EfficientNet-B1, DNA-d achieves 78.8\% top-1 accuracy, being evenly accurate but 1.4M smaller than EfficientNet-B1.
\textbf{Fourth}, with RandAugment \cite{cubuk2019randaugment}, DNA-c and DNA-d further achieve 78.1\% and 78.9\% top-1 accuracy.

\begin{table}
\centering
\caption{\small{Comparison of transfer learning performance of NAS models on CIFAR-10 and CIFAR-100. $x(x')$: $x$ is our re-implementational results with the officially released model; $x'$ is the accuracy reported by the original paper.}}\label{tab:cifar}
\vspace{-11pt}
\resizebox{0.4\textwidth}{!}{
\begin{tabular}{l|l|l}
Model & CIFAR-10 Acc & CIFAR-100 Acc\\
\hline\hline
MixNet-M\cite{Tan2019MixConvMD} & 97.9\% & 87.4\% \\
EfficientNet-B0 \cite{Tan2019EfficientNetRM} & 98.0\%(98.1\%)$^\dagger$ & 87.1\%(88.1\%)$^\dagger$ \\
DNA-c (ours) & 98.3\% &88.3\% \\
\end{tabular}
}
\vspace{-11pt}
\end{table}

\textbf{Architecture visualization.} Our searched architectures are visualized in Fig \ref{fig:arch_detail} in the appendix, from which we have several observations. \textbf{i)} Searched under no constraint, DNA-d tends to choose relatively expensive operations with high expansion rates and large kernel sizes, thus achieving the best performance. This verifies that our DNA is able to find optimal architectures in the search space. \textbf{ii)} Under the constraint of maximum parameter number, DNA-c tends to discard the operations with redundant channels to save the parameters. It also tends to select a lower expansion rate in the last blocks since the last blocks have more channels than the first blocks. \textbf{iii)} Under the constraint of maximum computational cost, DNA-b and DNA-a tend to select operations with fewer channels and lower expansion rates evenly in each block. The significant style difference of the high-performing architectures in Fig \ref{fig:arch_detail} proves DNA's architecture search ability.

\textbf{Model complexity vs. model accuracy.} Fig. \ref{fig:paramsota} compares the curve of \textbf{model size vs. accuracy} and \textbf{FLOPs vs. accuracy} for most recent NAS models. Our DNA models can achieve better accuracy with smaller model sizes and lower computation complexity than other NAS models.

\textbf{Transferability.} To test our searched architectures' transferability, we evaluate our searched architectures on two widely used transfer learning datasets, i.e., CIFAR-10 and CIFAR-100. The results are reported in Table \ref{tab:cifar}. As shown, our models maintain superiority after transferring.

\begin{table}
\setlength{\abovecaptionskip}{0.cm}
\setlength{\belowcaptionskip}{0.cm}
\centering
\caption{\small{Comparison between our scaled DNA models and EfficientNets on ImageNet.}}\label{tab:scaling}
\resizebox{0.4\textwidth}{!}{
 \begin{tabular}{l|c|c|c|c}
 model & input size & Acc@1 & Acc@5 &Params\\
 \hline\hline
 DNA-c & \multirow{2}{*}{224$\times$224} & \textbf{78.1\%} & \textbf{94.0\%} & 5.3M\\
 EfficientNet-B0 & & 76.3\% & 93.2\% & 5.3M\\
 \hline
 DNA-c1 & \multirow{2}{*}{240$\times$240} & \textbf{79.7\%} & \textbf{94.7\%} & 7.5M\\
 EfficientNet-B1 & & 78.8\% & 94.4\% & 7.8M\\
 \hline
 DNA-c2 & \multirow{2}{*}{260$\times$260} & \textbf{80.5\%} & \textbf{95.1\%} & 8.8M\\
 EfficientNet-B2 & & 79.8\% & 94.9\% & 9.2M\\
 \hline
 DNA-c3 & \multirow{2}{*}{300$\times$300} & \textbf{81.8\%} & \textbf{95.7\%} & 12M\\
 EfficientNet-B3 & & 81.1\% & 95.5\% & 12M\\
 \hline
 DNA-c4 & \multirow{2}{*}{380$\times$380} & \textbf{82.9\%} & \textbf{96.3\%} & 19M\\
 EfficientNet-B4 & & 82.6\% & 96.3\% & 19M\\
 \hline
 DNA-c5 & \multirow{2}{*}{456$\times$456} & \textbf{83.4\%} & 96.4\% & 30M\\
 EfficientNet-B5 & & 83.3\% & 96.7\% & 30M\\
 \end{tabular}
 }
 \vspace{-5pt}
\end{table}

\begin{table}
\scriptsize
\caption{\small{Top-1 accuracy of DNA+ on ImageNet. }}\label{tab:cascade}
\vspace{-6pt}
\centering
\resizebox{0.4\textwidth}{!}{
\begin{tabular}{m{0.4cm}|m{1.9cm}|m{2.1cm}|m{1.3cm}|m{0.8cm}}
\hline
     Gen & Teacher & Searched model & Searched model acc & Kendall tau\\\hline
     0 & --- & \textcolor{black}{\textbf{DNA+0}} (EfficientNet-b1)  & 76.3  & --- \\\hline
     1 & Scaled-\textcolor{black}{\textbf{DNA+0}} (EfficientNet-b7) & {\textbf{DNA+1}} ~~~~~~~~~~~(DNA-c)  & 78.1 & 0.64 \\\hline
     2 & Scaled-{\textbf{DNA+1}} (DNA-c1) & \textcolor{black}{\textbf{DNA+2}} & 78.3 & 0.66  \\\hline
     3 & Scaled-\textcolor{black}{\textbf{DNA+2}}  & DNA+3 & --- & 0.66  \\\hline
\end{tabular}
}
\vspace{-11pt}
\end{table}

\textbf{Model scalability.} As our searched architectures achieve remarkably high accuracy with significantly few parameters, to adequately explore our searched architecture's performance, we extend our DNA-c to the different model sizes by scaling up the depth, width, and input size of DNA-c simultaneously. Although it is better to search for an optimal scale strategy suitable for our DNA using a grid search as EfficientNet \cite{Tan2019EfficientNetRM} did, for simplicity, we just borrow the scaling strategy from \cite{Tan2019EfficientNetRM} without searching, putting our method at a disadvantage. The results are shown in Table \ref{tab:scaling}. We can see that our scaled architectures' accuracy surpassed the competing EfficientNet at all levels, even if the scaling strategy used might not fit our DNA perfectly due to the lack of strategy search. The left of Fig. 5 presents the scatter diagram. In summary, this comparison verifies the general features that our searched architecture has learned.

\textcolor{black}{\textbf{Fairer comparisons.} Please kindly note the importance of ensuring fairness when comparing our approach with the state-of-the-art method. There are two key points that merit thoughtful consideration in this regard: \textbf{Firstly}, during the validation process using the MBConv search space, our searched architectures undergo retraining from scratch without the utilization of knowledge distillation or fine-tuning. In contrast, recent practices often involve initializing the searched architecture with pre-trained weights from the supernet or employing larger networks for knowledge distillation. These disparities place our DNA approach at a potential disadvantage. To address this discrepancy, we have taken the proactive step of incorporating knowledge distillation into our retraining process in this revision. \textbf{Secondly}, given our emphasis on parameter efficiency, the DNA family conducts searches in the search phase with the number of parameters as the target. Therefore, it is essential to compare it with state-of-the-art methods that exhibit fair parameter efficiency.
We meticulously compare our method with the latest advanced techniques and thoroughly discuss the results. As demonstrated in Table \ref{tab:imagenet}, this fairer comparison reveals that DNA outperforms previous state-of-the-art approaches, such as FBNetV3 \cite{dai2021fbnetv3} (79.7\% vs. 78.4\%), AttentiveNAS \cite{wang2021attentivenas} (79.7\% vs. 77.3\%), AlphaNet \cite{wang2021alphanet} (79.7\% vs. 77.8\%), and EfficientNetV2 \cite{tan2021efficientnetv2} (79.7\% vs. 78.7\%), and OnceForAll\cite{cai2019once} (79.7\% vs. 76.4\%), by a significant margin.
}

\subsubsection{DNA+}
DNA+ is a multi-generation cascade of DNA.
As Table \ref{tab:cascade} shows, the searched models of Generations 0  and 1 are Efficient-b1 and DNA-c. The scaled searched models of Generations 0 and 1 are Efficient-b7 and DNA-c1, which further serve as the teachers of Generations 1 and 2, respectively. The searched models of Generations 2 and 3 are DNA+2 and DNA+3. For each generation, we evaluate the supernet's ranking ability with our ImageNet NAS Bench. The ranking Kendall tau along with searched model accuracies are reported. The accuracy leaped by 1.8\% at Generation 1. At Generation 2, the accuracy further improved by 0.2\%, and the ranking tau improved by about 3\% relatively. The ranking correlation reaches a plateau at Generation 3, as the teacher model may have reached the search space ceiling.

\begin{table}
\centering
\caption{\small{Comparison of vision transformers on ImageNet. The input size is $224\times 224$. \textcolor{black}{While our method exhibits superior performance to NASViT, we acknowledge that the DNA model size is larger compared to NASViT (5.6G vs. 1.9G). However, it is worth mentioning that this model is the largest available one in the NASViT paper and its GitHub page.}}}\label{tab:imagenet_vit}
\vspace{-6pt}
\resizebox{0.4\textwidth}{!}{
 \begin{tabular}{l|c|c}
 \toprule
 Model & FLOPs & Top-1 Acc (\%)\\
 \midrule
 \multicolumn{3}{l}{\textit{w/o distillation}}\\
 \midrule
 ViT-S/16~\pub{ICLR21}~\cite{dosovitskiy2021vit}            & 4.7G & 78.8 \\
 DeiT-S~\pub{ICML21}~\cite{touvron2020deit}                 & 4.7G & 79.9 \\
 T2T-ViT-14~\pub{ICCV21}~\cite{yuan2021tokens}              & 5.2G & 81.5\\
 Swin-T~\pub{ICCV21}~\cite{liu2021swin}                     & 4.5G & 81.3 \\
 \textcolor{black}{AutoFormer \pub{ICCV21} \cite{chen2021autoformer} }   & \textcolor{black}{5.1G} & \textcolor{black}{81.7}\\
 BossNAS-T $\bm\uparrow$\pub{ICCV21}~\cite{Li2021BossNAS_iccv}   & 5.7G & 81.6\\
 \midrule
 DNA-T                                             & 4.5G & 81.2\\
 DNA-T++                                           & 5.6G & 82.0\\
 \midrule
 \multicolumn{3}{l}{\textit{w/ distillation}}\\
 \midrule
 \textcolor{black}{GLiT-Small \pub{ICCV21}~\cite{chen2021glit}}   & \textcolor{black}{4.4G} & \textcolor{black}{80.5}\\
 \textcolor{black}{GLiT-Base \pub{ICCV21}~\cite{chen2021glit}}   & \textcolor{black}{17.0G} & \textcolor{black}{82.3}\\
 \midrule
 \textcolor{black}{ViTAS-Twins-S \pub{ECCV22} \cite{su2022vitas} }   & \textcolor{black}{3.0G} & \textcolor{black}{82.0}\\
 \textcolor{black}{ViTAS-Twins-B \pub{ECCV22} \cite{su2022vitas} }   & \textcolor{black}{8.8G} & \textcolor{black}{83.5}\\
 \textcolor{black}{ViTAS-DeiT-B \pub{ECCV22} \cite{su2022vitas} }   & \textcolor{black}{4.9G} & \textcolor{black}{80.2}\\
 \midrule
 \textcolor{black}{NASViT \pub{ICLR21} \cite{gong2021nasvit}}   & \textcolor{black}{1.9G} & \textcolor{black}{82.9}\\
 \midrule
 \textcolor{black}{Twins-PCPVT-S \pub{NeurIPS21} \cite{chu2021twins}}   & \textcolor{black}{3.8G} & \textcolor{black}{81.2} \\
 \textcolor{black}{Twins-PCPVT-B \pub{NeurIPS21} \cite{chu2021twins}}   & \textcolor{black}{6.7G} & \textcolor{black}{82.7} \\
 \midrule
 DeiT-S\alambic~\pub{ICML21}~\cite{touvron2020deit}         & 4.7G & 81.2 \\
 DS-ViT-L++ \pub{TPAMI22}~\cite{Li2022Ds_tpami}                                                & 5.6G & 83.0\\
 DNA-T \alambic                                              & 4.5G & 82.9 \\
 DNA-T++ \alambic                                           & 5.6G& 83.6\\
 \bottomrule
 \end{tabular}
 }
\vspace{-15pt}
\end{table}

\vspace{-11pt}
\subsection{Searching for Vision Transformers}
\vspace{5pt}

\textbf{Why do we select DNA and DNA++?} This section focuses on examining DNA and DNA++ because: First, the teacher we use is a CNN, and the architectures to be searched are ViTs, which have an architecture shift. The self-supervised learning technique adopted by DNA++ can solve this architecture shift problem. Second, DNA+ is time-consuming and has been tested in an efficient MBConv space.

\textbf{Supernet training.} We train each cell in the supernet for 20 epochs individually under the block-wise guidance of the teacher model. We use 0.33 and 0.0025 (for 256 total batch size) as the start learning rate for DNA and DNA++, respectively. For DNA, we use AdamW as our optimizer. The learning rate of DNA is reduced to its $0.9\times$ every $3$ epochs. For DNA++, we use cosine scheduler to reduce the learning rate and use LARS as the optimizer. It takes 1 and 2 days to train a typical supernet using 8 NVIDIA GTX 2080Ti GPUs for DNA and DNA++, respectively. With Algorithm \ref{alg:eval}, our cost of architecture rating is about 0.6 and 0.1 GPU days for DNA and DNA++. To search for the best model under certain constraints, we perform Algorithm \ref{alg:search} on CPUs, and the cost is less than one hour for both DNA and DNA++.

\textbf{Retraining searched models.}
To retrain our searched architectures on ImageNet, we use a similar setting with~\cite{touvron2020deit}, i.e., a batch size of 1,024, an AdamW optimizer, an initial learning rate of 1$e$-3, and a cosine learning rate scheduler. The weight decay is set to 5$e$-2. The architecture is trained for 300 epochs, with 5 epochs for learning rate warm-up. We use similar regularization and augmentation as DeiT \cite{touvron2020deit}, including label smoothing~\cite{szegedy2016rethinking}, stochastic depth~\cite{huang2016deep}, Cutmix~\cite{yun2019cutmix}, RandAugment~\cite{cubuk2020randaugment}, random erasing~\cite{zhong2020random}, and Mixup~\cite{zhang2018mixup}. Following~\cite{touvron2020deit}, we also explore training with KD, a standard technique widely used for training ViTs.

\subsubsection{DNA}

As shown in Table \ref{tab:imagenet_vit}, our DNA models achieve state-of-the-art results compared with existing models. More precisely, we have three critical findings. \textbf{First,} the performance of the frameworks searched by our DNA matches or even exceeds the performance of hand-designed frameworks. For example, at about 4.5G FLOPs of computation, our DNA-T outperforms ViT-S/16 \cite{dosovitskiy2021vit} and DeiT-S \cite{touvron2020deit}, close to Swin-T \cite{liu2021swin}, and slightly lower than T2T-ViT-14 \cite{yuan2021tokens}. Note that T2T-ViT-14 is computationally more expensive than our DNA-T (i.e., 5.2G vs. 4.5G). \textbf{Second}, our method maintains an advantage over hand-designed models with knowledge distillation (e.g., 82.8\% vs. 81.2\% for DNA-T vs. DeiT-S). \textbf{Last but not least}, we generously admit that our DNA-T is slightly inferior to the best NAS methods.

\textbf{Analysis.}
Two reasons can account for why our method does not hold a clear advantage over existing NAS methods. 
\textbf{First}, existing methods are more computationally expensive than our DNA-T (e.g., 5.7G FLOPs vs. 4.5G FLOPs for BossNAS-T vs. DNA-T).
\textbf{Second}, our DNA-T has an architecture shift between the teacher and the students. Specifically, our teacher is a CNN, and our students are ViTs. In the knowledge distillation process, operations similar to the teacher tend to have lower distillation losses. This makes the searched networks tend to be close to the teacher. This bias prevents the best architecture from being discovered. Fortunately, our DNA++ can solve this architecture shift, which will be presented in the next section.

\begin{figure*}
\setlength{\abovecaptionskip}{-0cm}
\setlength{\belowcaptionskip}{0.1cm}
 \centering
 \includegraphics[width=0.9\linewidth]{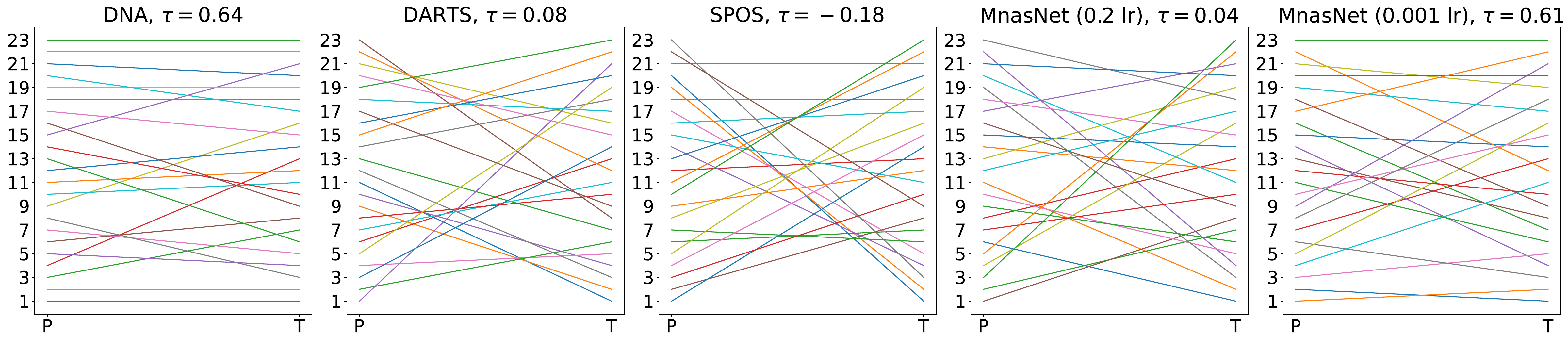}
 \caption{\small{Comparison of model ranking among \emph{DNA, DARTS \cite{Liu2018DARTSDA}, SPOS \cite{Guo2019SinglePO}, and MnasNet \cite{Tan2018MnasNetPN}} under two different hyper-parameters. The left and right y-axises represent the ``predicted model ranking'' (P) and ``actual model ranking'' (T), respectively.}}\label{fig:ranking}
 \vspace{-16pt}
\end{figure*}

\begin{figure*}
\setlength{\abovecaptionskip}{-0cm}
\setlength{\belowcaptionskip}{0.1cm}
 \centering
 \includegraphics[width=0.9\linewidth]{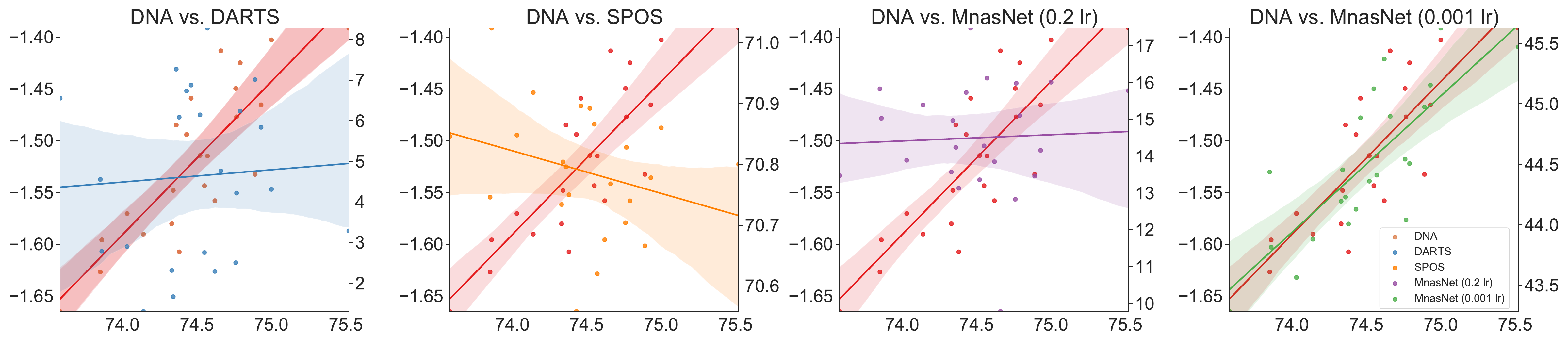}
 \caption{\small{Comparison of model ranking among \emph{DNA vs. DARTS \cite{Liu2018DARTSDA}}, \emph{DNA vs. SPOS \cite{Guo2019SinglePO}}, and \emph{DNA vs. MnasNet \cite{Tan2018MnasNetPN}} under two different hyper-parameters. The x-axises represent the ground-truth top-1 accuracies of different architectures on ImageNet, and the y-axises represent their losses based on the under-trained weights of different methods.}}\label{fig:rankingloss}
 \vspace{-16pt}
\end{figure*}

\begin{figure*}
\centering
\vspace{-3pt}
\subfigure[DNA (ours)]{\label{fig:progress_dna}\includegraphics[width=0.3\linewidth]{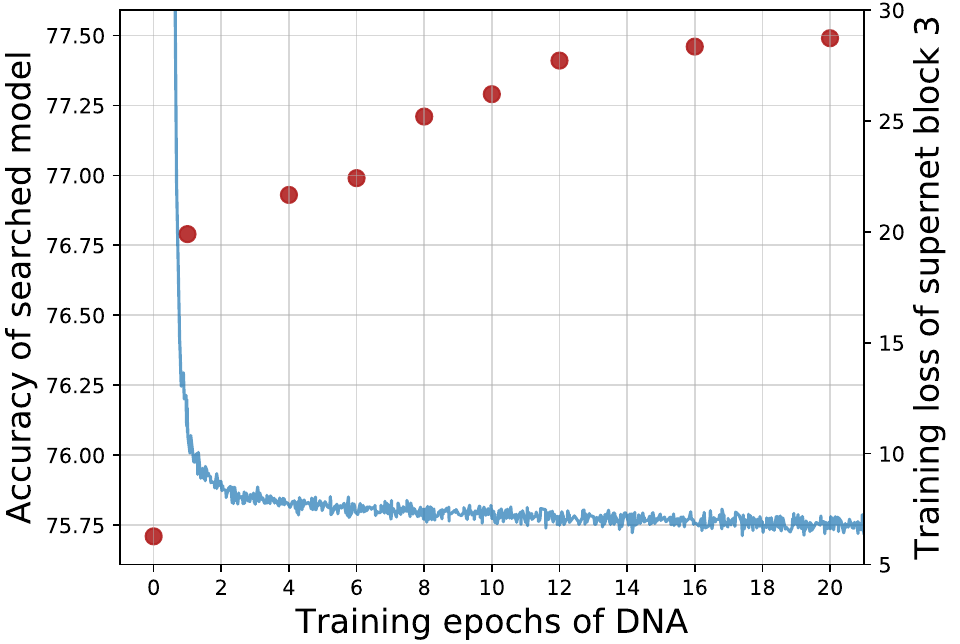}}
\subfigure[SPOS\cite{Guo2019SinglePO}]{\label{fig:progress_spos}\includegraphics[width=0.3\linewidth]{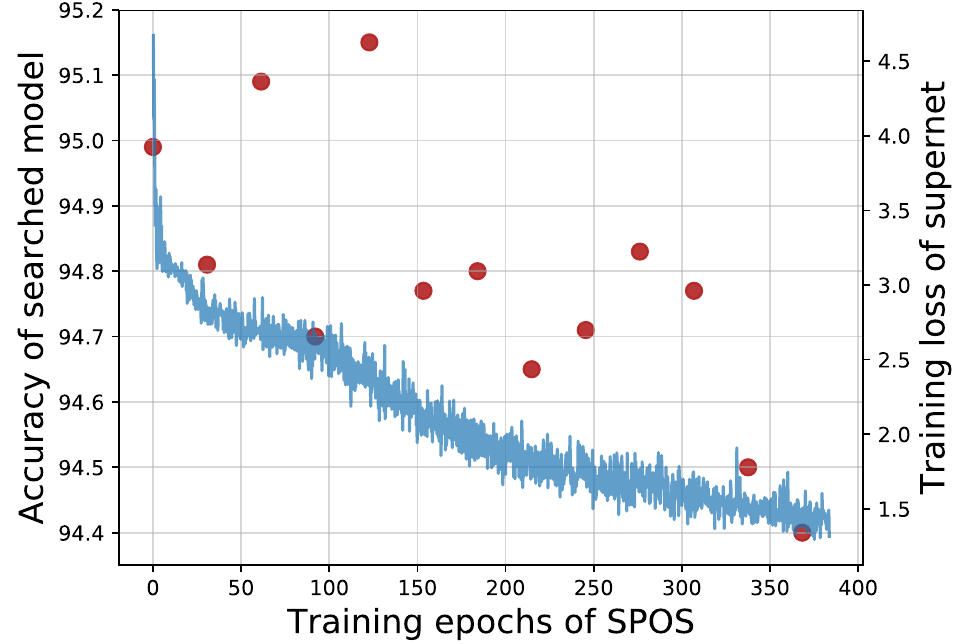}}
\subfigure[DARTS\cite{Liu2018DARTSDA}]{\label{fig:progress_darts}\includegraphics[width=0.3\linewidth]{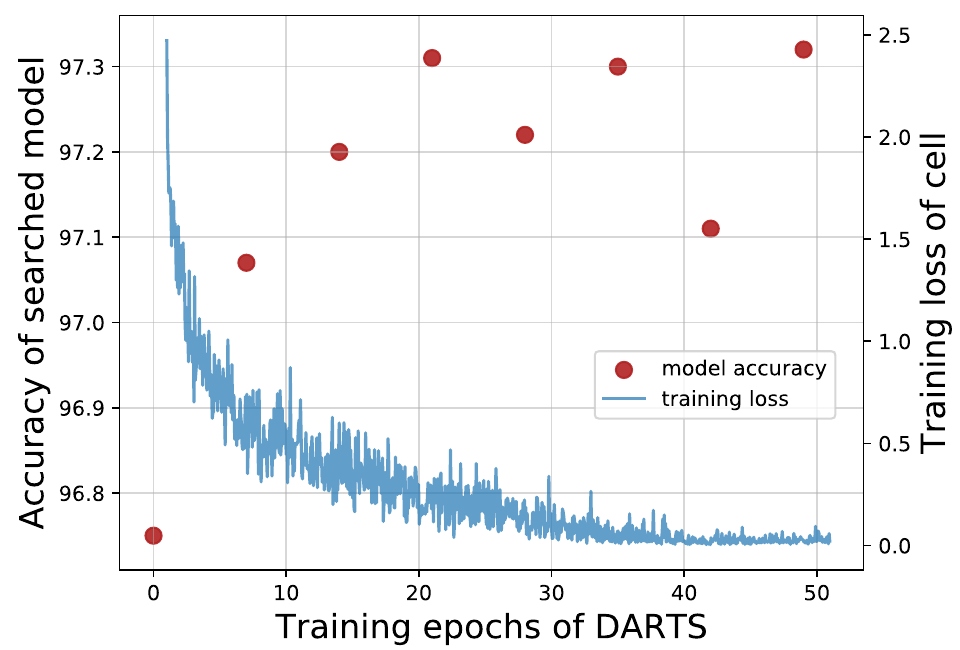}}
\vspace{-11pt}
\caption{\small{Top-1 accuracies of the searched models \emph{w.r.t} the training loss of the cell during the training progress on ImageNet.}}\label{fig:progress}
\vspace{-11pt}
\end{figure*}

\subsubsection{DNA++}

\textcolor{black}{We have conducted comprehensive quantitative comparisons between our DNA approach and these state-of-the-art NAS methods applied to ViTs. These comparisons are extensively presented in the findings outlined in Table \ref{tab:imagenet_vit}. The compelling results unequivocally demonstrate that our method consistently achieves a superior level of accuracy compared to these methods, even with similar FLOPs (e.g., 83.6\% vs. 80.5\% vs. 82.0\% vs. 81.7\% vs. 82.9\% vs. 82.7\% for DNA vs. GLIT \cite{chen2021glit} vs. VITAS \cite{su2022vitas} vs. AutoFormer \cite{chen2021autoformer} vs. NASViT \cite{gong2021nasvit} vs. Twins \cite{chu2021twins}).
}

\textbf{On the one hand}, our DNA-T++ achieves the best performance among all contrasting methods, including manual architectures \textcolor{black}{(e.g., Twins \cite{chu2021twins}, which proposes a smart attention mechanism for vision transformers)} and automatically searched architectures. 
\textcolor{black}{In terms of comparing our DNA approach with the Twins \cite{chu2021twins} method of similar FLOPs (e.g., 6.4G vs. 6.7G), we have obtained the following results: 83.6\% vs. 82.7\% for DNA-T++ vs. Twins-PCPVT-B.}
Moreover, DNA-T++ outperforms BossNAS, the current best NAS method for vision transformer search, by a clear margin (i.e., 82.0\% vs. 81.6\%). DNA++ even also significantly outperforms the current best-performing dynamic ViTs like DS-ViT-L++ \cite{Li2022Ds_tpami} (i.e., 83.6\% 83.0\%). \textcolor{black}{In terms of comparing our DNA approach with the NASViT \cite{gong2021nasvit} approach of similar size, we have obtained the following results: 83.6\% vs. 82.8\% for DNA-T++ vs. NASViT. While our method exhibits a superior performance to NASViT, we acknowledge that the DNA model size is larger compared to NASViT (5.6G vs. 1.9G). However, it is worth mentioning that this model is the largest available one in the NASViT paper and its GitHub page \footnote{\url{https://github.com/facebookresearch/NASViT}}.}
\textbf{On the other hand}, DNA-T++ far exceeds the performance of DNA-T (i.e., 82.0\% vs. 81.2\% without knowledge distillation and 83.6\% vs. 82.9\% with knowledge distillation), which is very encouraging to us.
The significant advantage of DNA-T++ over DNA-T confirms that self-supervised technology can solve the architecture shift problem between teachers and students in NAS.
This exciting result proves not only the effectiveness of DNA++, but also the robustness of our DNA family in the face of various corner cases in the network structure search.

\subsection{Effectiveness}\label{sec:effctiveness}
We prove our DNA's effectiveness and provide more insights into DNA by comparing it with recent NAS methods on model ranking, quantitative effectiveness/efficiency analysis, training stability, and feature visualization. Three methods, i.e., DARTS\cite{Liu2018DARTSDA}, SPOS\cite{Guo2019SinglePO}, and MnasNet\cite{Tan2018MnasNetPN} are chosen as the representatives of the three mainstreams of the current NAS: jointly-optimized weight-sharing methods, one-shot weight-sharing methods, and independent under-training methods. The experiments in this section (i.e., Section \ref{sec:effctiveness}) are conducted on our ImageNet NAS Bench.

\textbf{Model ranking}\footnote{\scriptsize
Our model ranking experiments are more comprehensive and fair than previous works like FairNAS \cite{Chu2019FairNAS_arxiv} and PCNAS \cite{anonymous2020improving}. FairNAS only samples models from the Pareto front, while PCNAS selects models from the last iterations of the search, introducing bias. In contrast, our model ranking includes architectures randomly sampled from the entire search space, ensuring a fairer comparison.
}. A main advantage of our DNA is reliable architecture rating, which means that the ranking of under-trained models (i.e., ``predicted model ranking'') can be an indicator of the actual ranking of these models (i.e., ``true model ranking''). To evaluate architecture rating, in Fig. \ref{fig:ranking}, we compare the ``predicted model ranking'' (P) and ``true model ranking'' (T) for DARTS, SPOS, MnasNet, and DNA, respectively.
All these methods are examined in the same MBConv search space (i.e., the smaller MBConv space described in Section \ref{sec:search_spaces}). 
The true accuracies and the ``true model ranking'' of these models are obtained from our ImageNet NAS Bench. We use the under-trained network weights of these four methods to get their ``predicted model ranking''. Each block in our DNA supernet is trained for 20 epochs, adding up to 120 epochs in total.
The SPOS supernet is also trained for 120 epochs following the official protocol \cite{Guo2019SinglePO}. 
We obtain the ``predicted model ranking'' for DNA and SPOS by evaluating the validation loss using supernet weights. The DARTS supernet is trained for 50 epochs because over-training DARTS could lead to a performance drop, as suggested by \cite{arber2020rdarts}. The ``predicted model ranking'' for DARTS is obtained by sorting the predicted joint probability of selecting different operation candidates for each architecture. 
As for MnasNet, each architecture is trained independently for five epochs following the official protocol \cite{Tan2018MnasNetPN}. 
The ``predicted model ranking'' is obtained by ranking the performance of these under-trained models. As the under-training setting of MnasNet is not provided in detail, we tried two different learning rates, i.e., 0.2 and 0.001, getting two very different results.

Fig. \ref{fig:ranking} presents the results. \textbf{First}, SPOS and DARTS fail to rank architectures correctly. This may be attributed to the bad generalization ability caused by the large search space, as Theorem \ref{theorem:generalization} suggested. \textbf{Second}, thanks to independent training, the model ranking of MnasNet is better than DARTS and SPOS. 
However, the cost of MnasNet is quite high. 
On the one hand, 8,000 out of the $10^{13}$ architectures are trained to obtain the Pareto front in MnasNet, which costs 288 TPU days \cite{Tan2018MnasNetPN}. 
On the other hand, it takes five epochs of ImageNet training to rate an architecture for MnasNet, while it only takes $\frac{120}{10^{10}}$ epochs of ImageNet training on average to rate an architecture for weight-sharing methods like SPOS and DNA.
Besides, the ranking result of MnasNet is highly hyperparameters-dependent (see the last two subfigures of Fig. \ref{fig:ranking}). The above two observations verify DNA's superiority in correctly and efficiently rating architectures, leading to effective and efficient architecture searches.

We also plot the true accuracies of different architectures on ImageNet \emph{w.r.t.} their under-trained models' losses for different NAS methods in Fig. \ref{fig:rankingloss}. Again, we can find that the losses obtained by our DNA are indicative of the real performance of the architectures. In contrast, the losses predicted by DARTS and SPOS are not indicative of the architectures' true accuracies.

\begin{table}
\setlength{\abovecaptionskip}{0.cm}
\setlength{\belowcaptionskip}{0.cm}
\centering
\caption{\small{Comparison of the effectiveness and efficiency of different NAS methods. The effectiveness is measured by Kendall Tau ($\tau$), Spearman Rho ($\rho$) and Pearson R ($R$), ranging from -1 to 1 (the higher, the better). The efficiency is measured by search cost (the lower, the better). (Gds: GPU days; Tds: TPU days)}}\label{tab:ranking}
\resizebox{0.45\textwidth}{!}{
\begin{tabular}{l|l|l|l|l}
Method & Search Cost & $\tau$ \cite{kendall1938new} & $\rho$ & $R$ \\\hline\hline
SPOS & \textbf{8.5 Gds} & -0.18 & -0.27 & -0.29  \\\hline
DARTS & 50 Gds & 0.08 & 0.14 & 0.06 \\\hline
MnasNet & 288 Tds & 0.04 / 0.61 & 0.07 / 0.77 & 0.05 / 0.78 \\\hline
DNA (Ours) & \textbf{8.5 Gds} & \textbf{0.64} & \textbf{0.82} &\textbf{0.84}
\end{tabular}
}
\vspace{-11pt}
\end{table}

\textbf{Quantitative analysis.}
We further provide a quantitative analysis of the effectiveness and efficiency of the four NAS methods. Following \cite{sciuto2019evaluating, anonymous2020nas}, we measure effectiveness by ranking correlation metrics, e.g., Kendall Tau ($\tau$) \cite{kendall1938new}, Spearman Rho ($\rho$) and Pearson R ($R$). All of the three metrics range from -1 to 1 with ``-1'' representing completely reversed ranking, ``1'' meaning entirely correct ranking, and ``0'' representing no correlation between rankings. We measure efficiency by search cost to measure.

Table \ref{tab:ranking} shows the results. \textbf{First}, although DARTS is renowned for high efficiency, it is less efficient than SPOS and DNA due to a large memory print. 
\textbf{Second}, DARTS, SPOS, and MnasNet with a 0.2 learning rate might be ineffective -- they are no better than or struggle to outperform random architecture selection with $\tau\approx0$). \textbf{Third}, MnasNet with a 0.001 learning rate achieves noticeably good effectiveness ($\tau=0.61$) at the cost of significantly low efficiency (288 TPU days). Our DNA obtains even better effectiveness ( $\tau=0.64$ \emph{vs.} $\tau=0.61$) with remarkably higher efficiency (8.5 GPU days \emph{vs.} 288 TPU days). \textbf{In summary}, DNA is the only member that achieves both high effectiveness and efficiency.

\begin{table}
\centering
\caption{\small{Robustness to different random seeds.}}\label{tab:seed}
\vspace{-11pt}
\begin{tabular}{l|c|c|c|c}
     Seed & 0 & 41  & 42\\\hline\hline
     Kendall tau to Acc & 0.64 & 0.64 & 0.64 \\\hline
     Kendall tau to seed 0 & - & 1.0 & 1.0 \\
\end{tabular}
\end{table}

\begin{table}
\centering
\caption{\small{Robustness to different training data amounts.}}
\vspace{-11pt}
\begin{tabular}{l|c|c|c|c}
     Dataset partition & full & 80\%  & 40\%\\\hline\hline
     Kendall tau to Acc & 0.64 & 0.63 & 0.64 \\\hline
     Kendall tau to full data & - & 0.98 & 1.0
\end{tabular}\label{tab:dataefficient}
\vspace{-11pt}
\end{table}

\begin{figure}[b]
    \centering
    \vspace{-11pt}
    \includegraphics[width=0.4\textwidth]{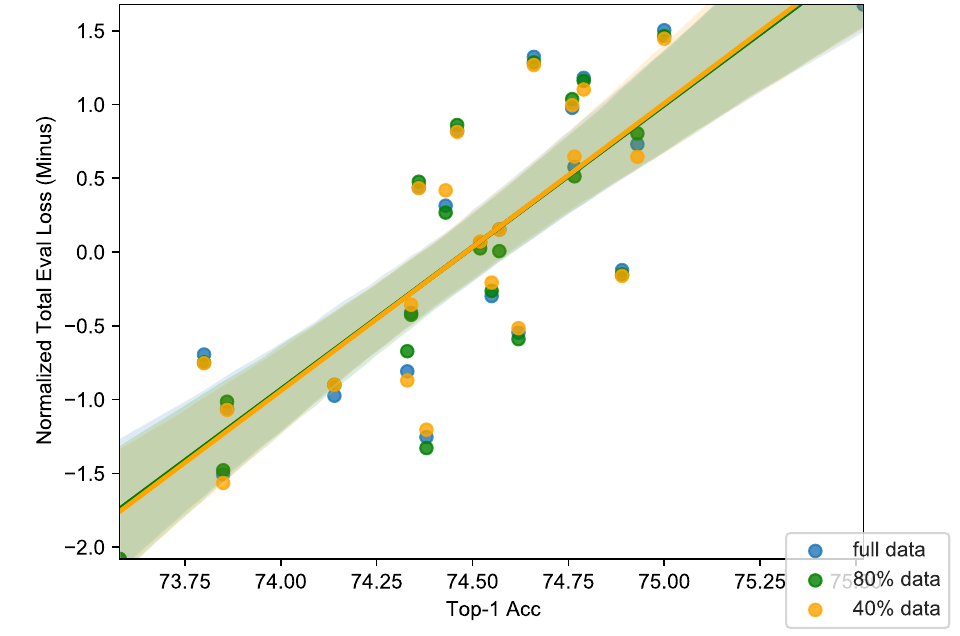}
    \vspace{-11pt}
    \caption{\small{Model ranking of different data amounts.}}\label{fig:reproduce}
    \vspace{-6pt}
\end{figure}

\begin{figure*}
\centering
\includegraphics[width=0.75\linewidth]{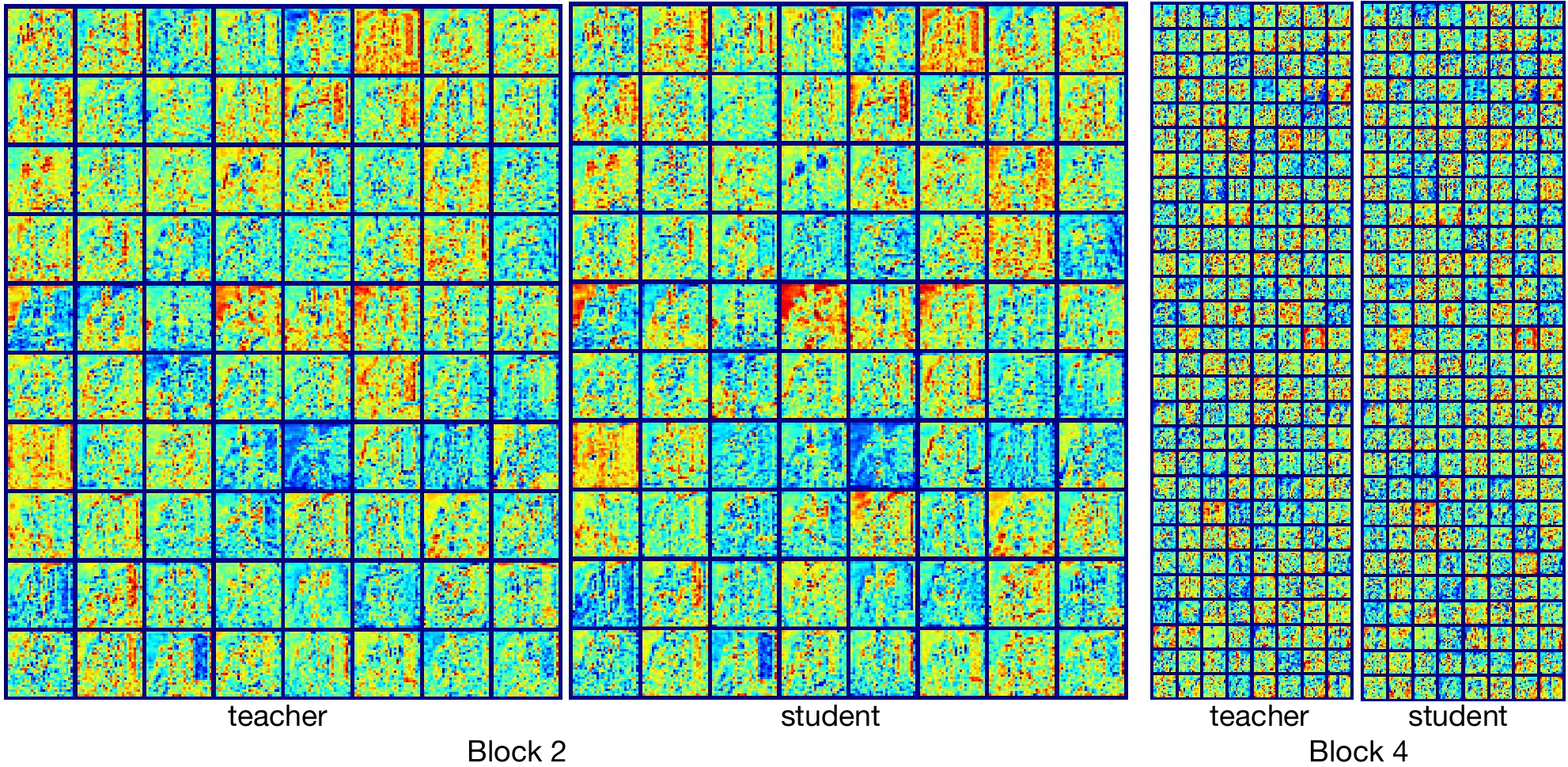}
\vspace{-9pt}
\caption{\small{Comparison of the feature maps between the teacher and student of two blocks.}}\label{fig:feature}
\vspace{-11pt}
\end{figure*}

\textbf{Training stability.} 
To further validate the search effectiveness, we examine whether DNA could consistently find better architectures as the supernet training goes on. We compare DNA with two supernet-based methods, i.e., SPOS and DARTS. For each method, we pick 8 intermediate checkpoints in the supernet training progress and search for the best architecture at each checkpoint. We then randomly re-initialize the network weights of the architecture and retrain them from scratch to convergence to obtain their real performance.
Fig. \ref{fig:progress} shows the results. 
\textbf{First}, for SPOS and DARTS, the performance of the searched architecture fluctuates randomly as the supernet training goes, although the training loss of the supernet keeps decreasing. 
This makes good architectures almost inaccessible because we cannot decide which epoch of the supernet checkpoint is optimal. \textbf{Second}, the accuracy of our searched architectures increases progressively (i.e., from 75.7\% to 77.5\%) as the supernet training goes on until convergence between the 16th and 20th epochs. This implies that we can use the converged supernet (e.g., the last epoch) to search for the architecture. Note that the accuracy increases rapidly in the early stage with the same tendency of training loss decreasing, which evidences a correlation between the accuracy of the searched architecture and the loss of the supernet. The stably increasing performance of searched architecture proves the stability of DNA.

\textbf{Reproducibility/Robustness.} NAS is worrying for being difficult to reproduce \cite{anonymous2020nas}. For instance, the seed is critical in reproducing DARTS results, possibly because of their sensitivity to initial weights. To evaluate the reproducibility of DNA, we conducted repeated experiments with two ablation types: i) different random seeds and ii) different data amounts. For seeds, we choose \{0, 41, 42\} that affect the supernet weight initialization and random path sampling during the search. Table \ref{tab:seed} shows that model ranking is stable for different seeds, proving the high reproducibility and robustness of DNA. For data amounts, we randomly sampled a subset of the ImageNet training set as the new training set and kept the total training steps unchanged. Table \ref{tab:dataefficient} shows that, with the training data being reduced to 80\% and further to 40\%, DNA's model ranking remains stable. A detailed model ranking comparison of different data amounts is shown in Fig. \ref{fig:reproduce}. To remove the effect of variance in scale and give a more intuitive comparison, we perform standardization over the total losses. Every 3 points are associated with the same accuracy. As shown, the points of every single model are closely bonded, and the regression lines of the three data amounts coincide with each other. This proves DNA's robustness to the data amount, implying that DNA can be a data-efficient few-shot NAS method.

\textbf{Feature map visualization.}
Several feature maps (Block 2 and 4, Epoch 16) of the teacher and student are visualized in Fig. \ref{fig:feature}. As shown, our student supernet, although with only $1/3$ of the layer numbers of the teacher model, can imitate the teacher very well. The textures are incredibly close at every channel, even on highly abstracted $14\times14$ feature maps of Block 4, proving the effectiveness of DNA in distilling knowledge from the teacher.

\subsection{Ablation Study}\label{sec:ablation}

\textbf{Distillation strategy.}
We use a parallel distillation strategy in our DNA, i.e., we distill different blocks of the student supernet \emph{concurrently} since the input and supervision of each student block are the feature maps of the teacher that can be pre-computed. We compare our parallel strategy to two progressive distillation strategies. A progressive strategy means training and searching each block of the student supernet \emph{one by one progressively}, in which: one trains Block 1 of the supernet and performs architecture search in Block 1 to obtain $m$ optimal blocky architectures; then, with the $m$ blocky architectures, one trains Block 2 of the supernet and perform architecture search in Block \textbf{1\&2} to obtain $m$ optimal blocky architectures; the cycle is repeated until the search is done. The two progressive strategies differ in that: \textbf{(A)} when training the $i$-th block of the supernet, all the $i$ blocks are trained from scratch; \textbf{(B)} when training the $i$-th block of the supernet, only the $i$-th block is trained from scratch with the weights of the previous ($i$-$1$) blocks of the supernet being frozen. The results in Table \ref{tab:strategy} prove the superiority of our strategy.

\begin{table}[t]
\setlength{\abovecaptionskip}{0.cm}
\setlength{\belowcaptionskip}{-0.cm}
\centering
\caption{\small{Component analysis of DNA. \textbf{Cons.:} constraint.}}\label{tab:strategy}
\resizebox{0.45\textwidth}{!}{
\begin{tabular}{c|c|c|c|c|c|c}
Strategy & name & Cell & Cons. &Params &Acc@1 &Acc@5\\\hline\hline
\multirow{2}{*}{progressive} & A & 1 & & 5.18M &77.0\% &93.34\%\\
 & B & 1 & & 5.58M &77.15\% &93.51\%\\\hline
\multirow{4}{*}{concurrent}&\multirow{4}{*}{Our} & 1 & & 5.69M &77.49\% &93.68\%\\
& &3 & & 6.26M &77.84\% &93.74\%\\
& & 1 &\checkmark & 5.09M &77.21\%	&93.50\%\\
& &3 &\checkmark & 5.28M &77.38\% &93.60\%\\
\end{tabular}
}
\vspace{-11pt}
\end{table}
\begin{table}
\centering
\caption{\small{DNA with different teachers. Note that all the searched models are retrained from scratch without any supervision of the teacher. $\mathbf{^\dagger}$: As the inputs and supervisions of a student block are precisely the teacher model's internal feature maps, EfficientNet-B7 is tested with an input size of $224\times224$ to adapt to the student.}}\label{tab:selfdistill}
\vspace{-11pt}
\begin{tabular}{l|l|c|c|c}
Role & Model & Params & Acc@1 & Acc@5\\
\hline\hline
teacher & EfficientNet-B0 & 5.28M & 76.3\% & 93.2\% \\
student & DNA-from-B0 & 5.27M & 77.8\% & 93.7\% \\
\hline\hline
teacher & EfficientNet-B7 & 66.4M & 77.8\%$^\dagger$ & 93.8\%$^\dagger$ \\
student & DNA-c & 5.28M & 77.8\% & 93.7\% \\\hline
student & DNA-c7 (224$\times$224) & 64.9M & 79.9\% & 94.9\% \\
\end{tabular}
\vspace{-11pt}
\end{table}

\textbf{Single cell \emph{vs.} multi cells.} To examine the multi-cell search's impact, we perform DNA with a single cell in each block for comparison. As shown in Table \ref{tab:strategy}, the multi-cell search improves the top-1 accuracy by 0.2\% under the constraint of 5.3M parameters and by 0.3\% under no constraint. Note that the single-cell case obtained an architecture with fewer parameters under the constraint of 5.3M parameters, which might be ascribed to the relatively lower variability of the channel and layer numbers.

\textbf{Teacher dependency.} As knowledge distillation is employed in our search stage (but \textbf{NOT} in the retraining stage), we must clarify whether the capacity of our searched architecture is limited by that of a teacher. \textbf{First}, we replace our senior teacher (i.e., EfficientNet-B7) with a junior teacher (i.e., EfficientNet-B0) to supervise architecture search under the same constraint. The result is shown in Table \ref{tab:selfdistill}. Surprisingly, the performance of the searched architecture guided by EfficientNet-B0 is almost the same as that guided by EfficientNet-B7, which indicates that the performance of our DNA does not necessarily rely on a high-performing teacher. More importantly, we can find that DNA-from-B0 searched by using EfficientNet-B0 as a teacher significantly outperforms its teacher by 1.5\% with the same model size (5.27M \emph{vs.} 5.28M), which proves that the performance of our architecture distillation is not restricted by the performance of the teacher. As implied, we could improve the capacity of any architecture by self-distilling the architecture. \textbf{Second}, as our searched architecture DNA-c achieves the same accuracy as its teacher with significantly fewer parameters (5.28M \emph{vs.} 66.4M), we scale our DNA-c to the similar model size as the teacher following \cite{Tan2019EfficientNetRM}. As the inputs and supervisions of a student block are exactly the internal feature maps of the teacher model, both EfficientNet-B7 and DNA-c7 in Table \ref{tab:selfdistill} are tested with an input size of $224\times224$. Besides, DNA-c7 in Table \ref{tab:selfdistill} is also trained with an input size of $224\times224$. Remarkably, the scaled architecture (i.e., DNA-c7) outperforms the teacher (i.e., EfficientNet-B7) by 2.1\%, demonstrating the practicability and scalability of our DNA. \textbf{Third} (\textbf{last but not least}), although the performance of our architecture distillation is not restricted by the performance of the teacher, recurrently scaling our student to be a new teacher can search for better and better architectures, as suggested by our DNA+. This indicates that we can randomly select an existing model as our teacher without allocating too much effort. Then we can perform recurrent NAS to obtain the optimal architecture.

\begin{figure}
\centering
\subfigure[\small{VOC mIoU vs. FLOPs}]{\label{fig:voc_flops}\includegraphics[width=0.48\linewidth]{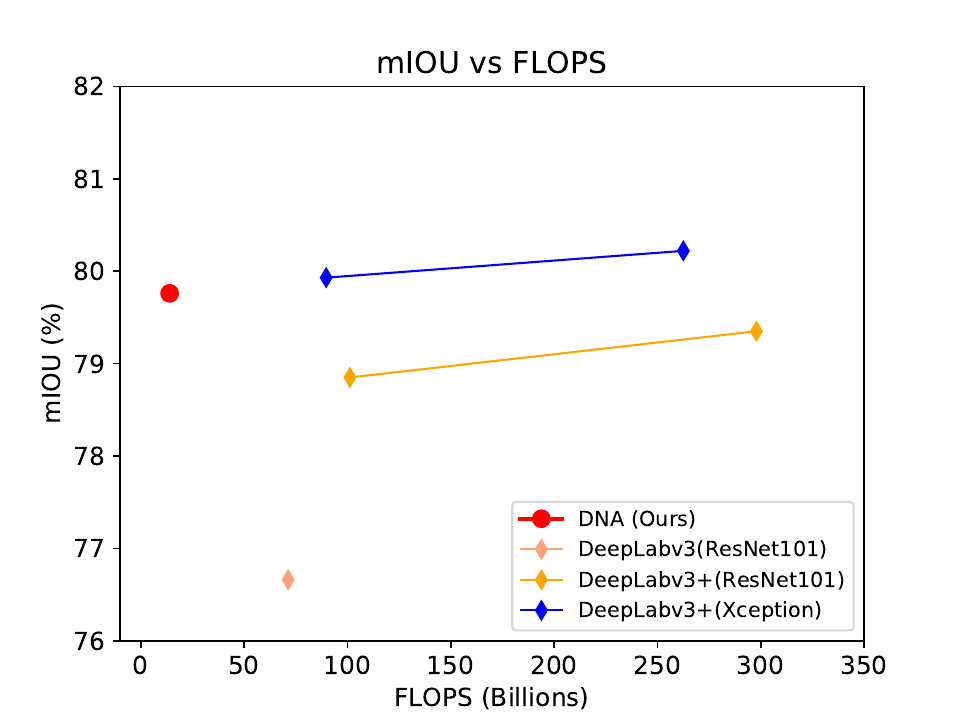}}
\subfigure[\small{VOC mIOU vs. Params}]{\label{fig:voc_params}\includegraphics[width=0.48\linewidth]{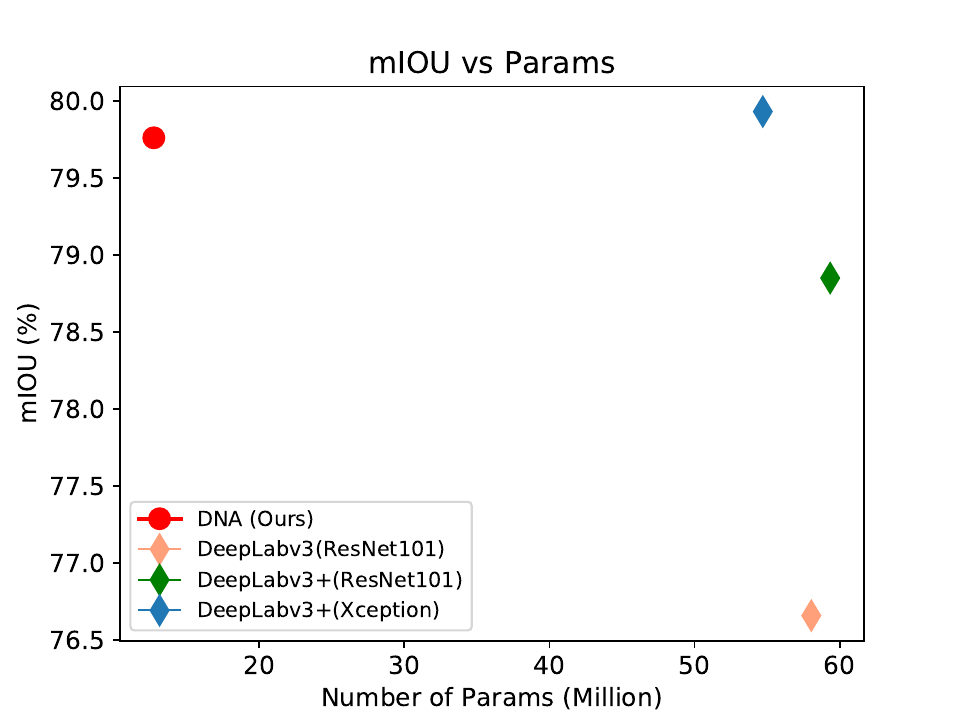}}
\vspace{-11pt}
\caption{\small{Trade-off between model accuracy and model complexity on PASCAL VOC12 segmentation. \textbf{Left:} model accuracy \emph{vs.} FLOPs; \textbf{Right:} model accuracy \emph{vs.} parameter numbers. }}\label{fig:voc}
\vspace{-11pt}
\end{figure}

\subsection{Semantic Segmentation on VOC12 and ADE20K}

\textbf{PASCAL VOC12.} As handcrafted architectures often exhibit good generalization beyond ImageNet/CIFAR classification, we examine the universality of our searched architectures. We take semantic segmentation on the challenging PASCAL VOC12 benchmark \cite{Everingham2015Pascal_ijcv} as an example. This dataset includes 21 categories, with 10,582 training images (\cite{Everingham2015Pascal_ijcv,Hariharan2011Semantic_iccv}) and 1,449 pixel-level labeled images in the validation set. We measure accuracy using the standard pixel intersection-over-union metric averaged across the 21 categories (mIOU). We replaced the backbone of DeepLabv3+ \cite{Chen2018Encoder_eccv}, a top-performing segmentation model, with our DNA-c3, forming DNA-Seg. Other protocols align with \cite{Chen2018Encoder_eccv}.

\textbf{Complexity vs. accuracy.}
Examining the performance on the test set (i.e., fighting for a champion of the leaderboard) of VOC12 segmentation usually includes many unknown tricks (e.g., employing multi-scale inputs during the evaluation, adding left-right flipped inputs, pre-training on MS-COCO dataset \cite{Lin2014Microsoft_eccv}, pre-training on JFT-300M dataset \cite{Sun2017Revisiting_iccv}, and ensembling multi-model), making the model complexity of different methods unclear. To avoid this ambiguity, we compare the complexity and accuracies of different models on the validation set of VOC12 segmentation for a fair comparison. The results are presented in Fig. \ref{fig:voc}. Under similar accuracy constraint, our DNA-Seg use 21$\times$ fewer FLOPs and 4.5$\times$ fewer parameters than the recent DeepLabv3+(ResNet101), 6$\times$ fewer FLOPs and 4.2$\times$ fewer parameters than the recent DeepLabv3+(Xception). In particular, with a single-model and single test-time scale, our DNA-Seg achieves state-of-the-art 79.76\% validation IoU with 12.75M parameters and 14.01 FLOPs.

\begin{table}[h]
\vspace{-5pt}
\scriptsize
\begin{center}
\caption{\small{\textcolor{black}{Segmentation tasks with CNNs on ADE20K.}}}
\vspace{-6pt}
\textcolor{black}{
\begin{tabular}{c|ccc|c|c}
\toprule
\multirow{2}{*}{Backbone} & \multicolumn{3}{|c|}{EfficientNet \cite{Tan2019EfficientNetRM}} & \multirow{2}{*}{FBNetV5 \cite{wu2021fbnetv5}} &   \multirow{2}{*}{\textbf{DNA-d}} \\
& -B0 \cite{Tan2019EfficientNetRM} & -B1 \cite{Tan2019EfficientNetRM} & -B2 \cite{Tan2019EfficientNetRM} &  &   \\
\midrule
mIoU (\%) & 38.9 & 39.2 & 40.4 & 40.4 &\textbf{40.7}   \\
\bottomrule
\end{tabular}
}
\label{tab:ade20k_r2_plus}
\end{center}
\vspace{-11pt}
\end{table}

\begin{table}[h]
\vspace{-5pt}
\scriptsize
\begin{center}
\caption{\small{\textcolor{black}{Segmentation tasks with ViTs on ADE20K.}}}
\vspace{-6pt}
\textcolor{black}{
\begin{tabular}{c|ccc}
\toprule
Backbone & NASViT \cite{gong2021nasvit} & S3 \cite{chen2021searching} &   DNA-T++  \\
\midrule
mIoU (\%) &  41.4 & 46.3 & 47.3  \\
\bottomrule
\end{tabular}
}
\label{tab:ade20k_vit}
\end{center}
\vspace{-11pt}
\end{table}

\noindent
\textcolor{black}{
\textbf{ADE20K.} In addition to PASCAL VOC12, we further include experiments on ADE20K \cite{zhou2019semantic}, expanding the scope of our evaluation. For the segmentation task, we have employed the Atrous Spatial Pyramid Pooling framework \cite{chen2017rethinking} as our segmenter. The results, as presented in Table \ref{tab:ade20k_r2_plus}, clearly demonstrate that our DNA method surpasses state-of-the-art efficient backbone methods by a significant margin. For instance, the performance comparison reveals that our DNA method achieves superior results (40.7\%) compared to EfficientNet-B0 \cite{Tan2019EfficientNetRM} (38.9\%), EfficientNet-B1 \cite{Tan2019EfficientNetRM} (39.2\%), and EfficientNet-B2 \cite{Tan2019EfficientNetRM} (40.4\%). Furthermore, it is worth noting that even the accuracy reported by FBNetV5 \cite{wu2021fbnetv5} (40.4\%) on ADE20K is inferior to the performance of our DNA method. These findings substantiate the effectiveness and superiority of our approach.
}

\textcolor{black}{
We further benchmark our approach against the two distinguished ViT-based methods, specifically NASViT \cite{gong2021nasvit} and S3 \cite{chen2021searching}. We have employed DNA-T++ as our selected backbone architecture. Under the UperNet configuration (which is identical to the setting used in S3), our model has demonstrated an exceptional accuracy rate of 47.3\% on the ADE20k dataset, outperforming both NASViT (41.4\%) and S3 (46.27\%). It is of utmost importance to emphasize that our model exhibits a larger scale in comparison to NASViT (5.6G vs. 1.9G), see Table \ref{tab:ade20k_vit}.
However, it's worth mentioning that this model is already the largest available in its paper and on its GitHub page.
}

\subsection{\textcolor{black}{Object Detection on MSCOCO}}

\begin{table}[h]
\vspace{-5pt}
\scriptsize
\begin{center}
\caption{\small{\textcolor{black}{Object detection results of Faster RCNN with different backbones on COCO.}}}
\vspace{-6pt}
\textcolor{black}{
\begin{tabular}{c|ccc}
\toprule
Backbone & EfficientNetB0 \cite{Tan2019EfficientNetRM} & FBNetV3-A \cite{dai2021fbnetv3} &  \textbf{DNA-d} \\
\midrule
mAP (\%) & 30.2 & 30.5 & \textbf{32.3}   \\
\bottomrule
\end{tabular}
}
\label{tab:coco_r2_plus}
\end{center}
\vspace{-11pt}
\end{table}

\textcolor{black}{In order to further assess the transferability of the searched models across various tasks, we employ DNA-d as a substitute for the backbone feature extractor in Faster R-CNN, specifically using the conv4 (C4) backbone, following the paper of FBNetV3 \cite{dai2021fbnetv3}. We then proceed to compare the performance of this configuration with other models on the COCO detection dataset. 
We have conducted a comprehensive comparison of our method with state-of-the-art approaches such as FBNetV3 \cite{dai2021fbnetv3} and EfficientNet-B0 (selected based on the recommendation of the FBNetV3 paper \cite{dai2021fbnetv3}) on the MSCOCO dataset \cite{Lin2014Microsoft_eccv}. The results, as depicted in Table \ref{tab:coco_r2_plus}, unequivocally reveal the substantial superiority of our method over these esteemed competitors, surpassing them by a clear margin (32.3\% vs. 30.5\%). 
}

\section{Conclusion}\label{sect:conclusion}

In this paper, we employ a tool for generalization boundedness to link weight-sharing NAS's inefficiency to unreliable architecture ratings due to a vast search space. To address this issue, we modularize the search space into blocks and apply distilling neural architecture techniques. We explore three block-wise learning methods: supervised learning (DNA), progressive learning (DNA+), and self-supervised learning (DNA++).
Our DNA family evaluates all candidate architectures, a significant advancement over prior methods restricted to smaller sub-search spaces via heuristic algorithms. Additionally, our approach enables the search for architectures with varying depths and widths under specified computational constraints.
Recognizing the pivotal role of architecture rating in NAS, we provide extensive empirical results to scrutinize this aspect. Lastly, our method attains state-of-the-art results across various tasks. Future work will extend the application of our DNA method to NLP, 3D architectures \cite{wang2023perf,wang2023sparsenerf,chen2023scenedreamer}, and generative models \cite{wang2022traditional,wang2022stylelight,chen2022text2light}.

\section{\textcolor{black}{Highlight and Outlook}}

\textcolor{black}{NAS has been a significant area of research and development in deep learning and AI. Nevertheless, it is plausible that the extent of attention garnered by NAS has undergone evolution or variation over time. Several conceivable factors may elucidate why NAS has witnessed diminished prominence during certain junctures or within specific contexts: \textbf{(a) Architecture Ranking:} we believe that the suboptimal architecture ranking is a key factor contributing to the reduced attention on NAS. Specifically, there exists a significant apprehension pertaining to the efficacy of the search process, making the researchers and practitioners worrying. To be more precise, this concern centers on the correlation between program performance during the search phase, specifically on small proxy tasks, and their subsequent performance on the final task. In the absence of a robust linkage between these two tasks, researchers may adopt a cautious stance towards NAS. \textbf{(b) Maturation of Architectures:} several deep learning architectures, such as CNNs and ViTs, had matured and demonstrated strong performance across a wide range of applications. Without resolving the architecture ranking issue, there may be less urgency to search for new architectures. \textbf{(c) Computational Cost:} NAS is computationally expensive and time-consuming, requiring significant computational resources to search for optimal neural network architectures. Researchers and organizations might prioritize more efficient approaches to model design and optimization unless the architecture ranking issue is addressed. \textbf{(d) Foundation Models:} foundation models \cite{bommasani2021opportunities} have become popular techniques for many AI tasks. These approaches allow practitioners to leverage existing architectures and adapt them to specific tasks without the need for extensive architecture search, unless the architecture ranking problem raised in our paper is resolved.}

\textcolor{black}{As evident, the resolution of the architecture ranking issue holds the capacity to alleviate all the previously mentioned challenges. Thus, the central focus of this paper is to diligently address the architecture ranking challenge with the aspiration of bolstering the efficacy of NAS methods.}

\section*{Acknowledgement}
This work is supported by the following grants: National Key R\&D Program of China under Grant No. 2021ZD0111601, National Natural Science Foundation of China (NSFC) under Grant No.61836012, 62006255, 62325605, GuangDong Basic and Applied Basic Research Foundation under Grant No. 2023A1515011374, GuangDong Province Key Laboratory of Information Security Technology.


\bibliographystyle{IEEEtran}
\bibliography{IEEEabrv,egbib}

\begin{IEEEbiography}[{\includegraphics[width=1in,height=1.25in,clip,keepaspectratio]{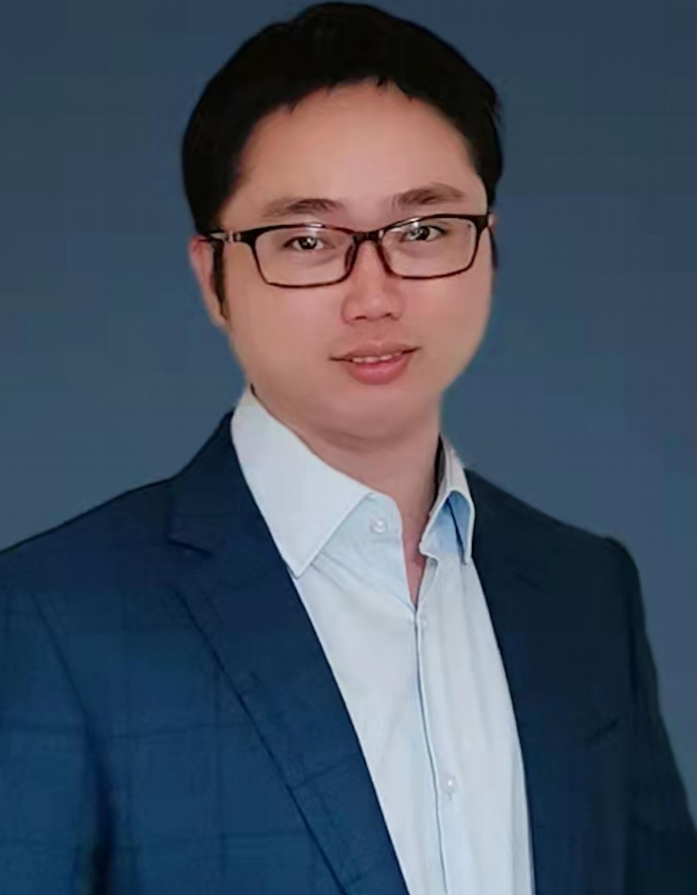}}]{Guangrun Wang} is currently a Postdoctoral Researcher in the Department of Engineering Science at the University of Oxford. He received two B.E. degrees and one Ph.D. degree from SYSU in 2014 and 2020. He was a visiting scholar at the Chinese University of Hong Kong (CUHK). His research interest is machine learning. He is a Distinguished Senior Program Committee member for IJCAI, an outstanding reviewer (six times) for ICLR, NeurIPS, and ICCV. He is the recipient of the 2018 Pattern Recognition Best Paper Award, two ESI Highly Cited Papers, Top Chinese Rising Stars in Artificial Intelligence, and Wu Wen-Jun Best Doctoral Dissertation.
\end{IEEEbiography}

\begin{IEEEbiography}[{\includegraphics[width=1in,height=1.25in,clip,keepaspectratio]{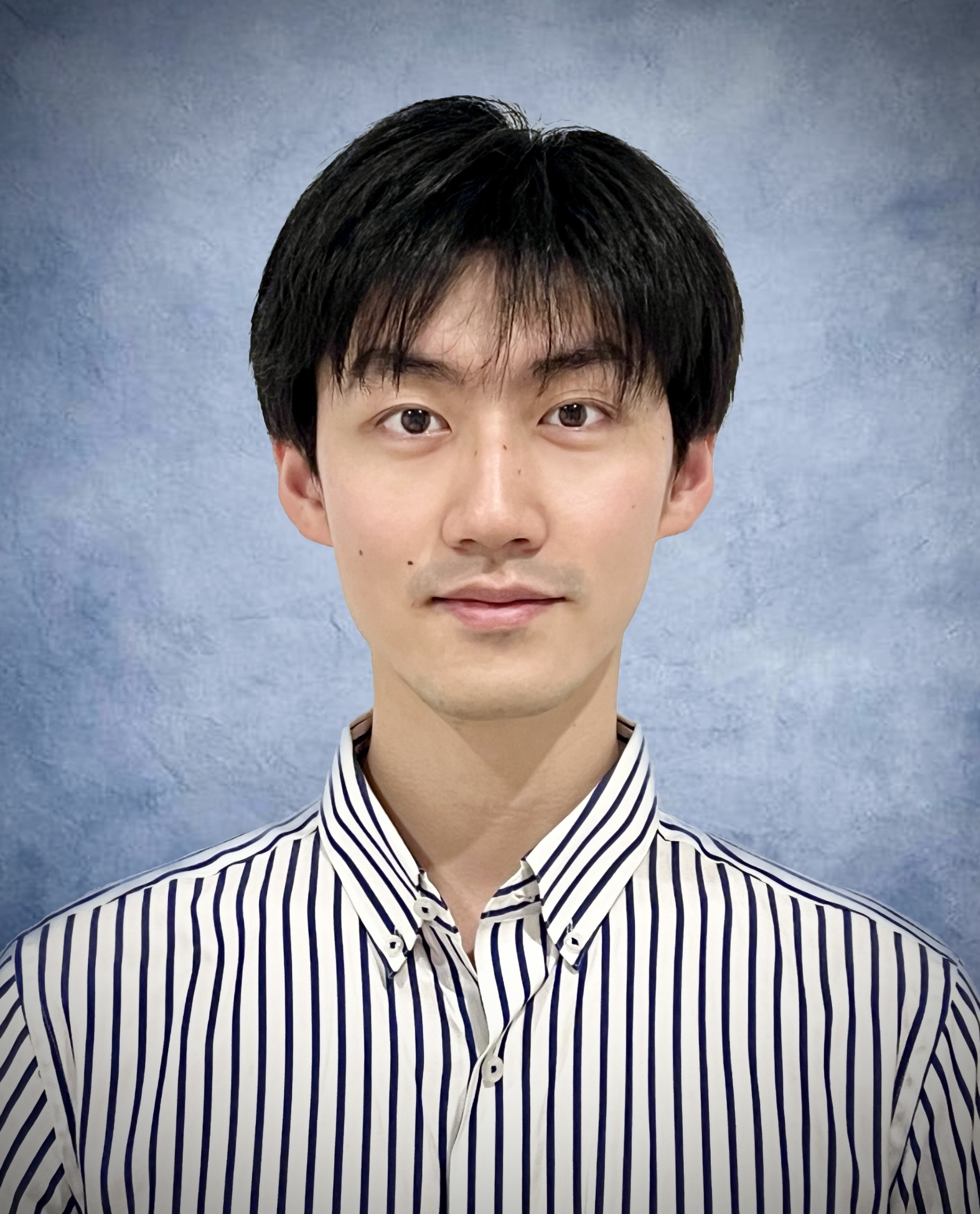}}]{Changlin Li} is a Postdoctoral Researcher in ReLER Lab, Australian Artificial Intelligence Institute, University of Technology Sydney (UTS). He received his Ph.D. degree from UTS in 2023.
Prior to his Ph.D. study, he received his B.E. degree in Computer Science in 2019, from University of Science and Technology of China. He currently serves as a reviewer of CVPR, ICCV, ECCV, ICML, NeurIPS, T-PAMI, T-IP, etc. His main research ambition is to explore more efficient and more intelligent neural network architectures. He is also interested in computer vision tasks such as activity recognition and has won first place in the TRECVID ActEV 2019 grand challenge.\end{IEEEbiography}

\begin{IEEEbiography}[{\includegraphics[width=1in,height=1.25in,clip,keepaspectratio]{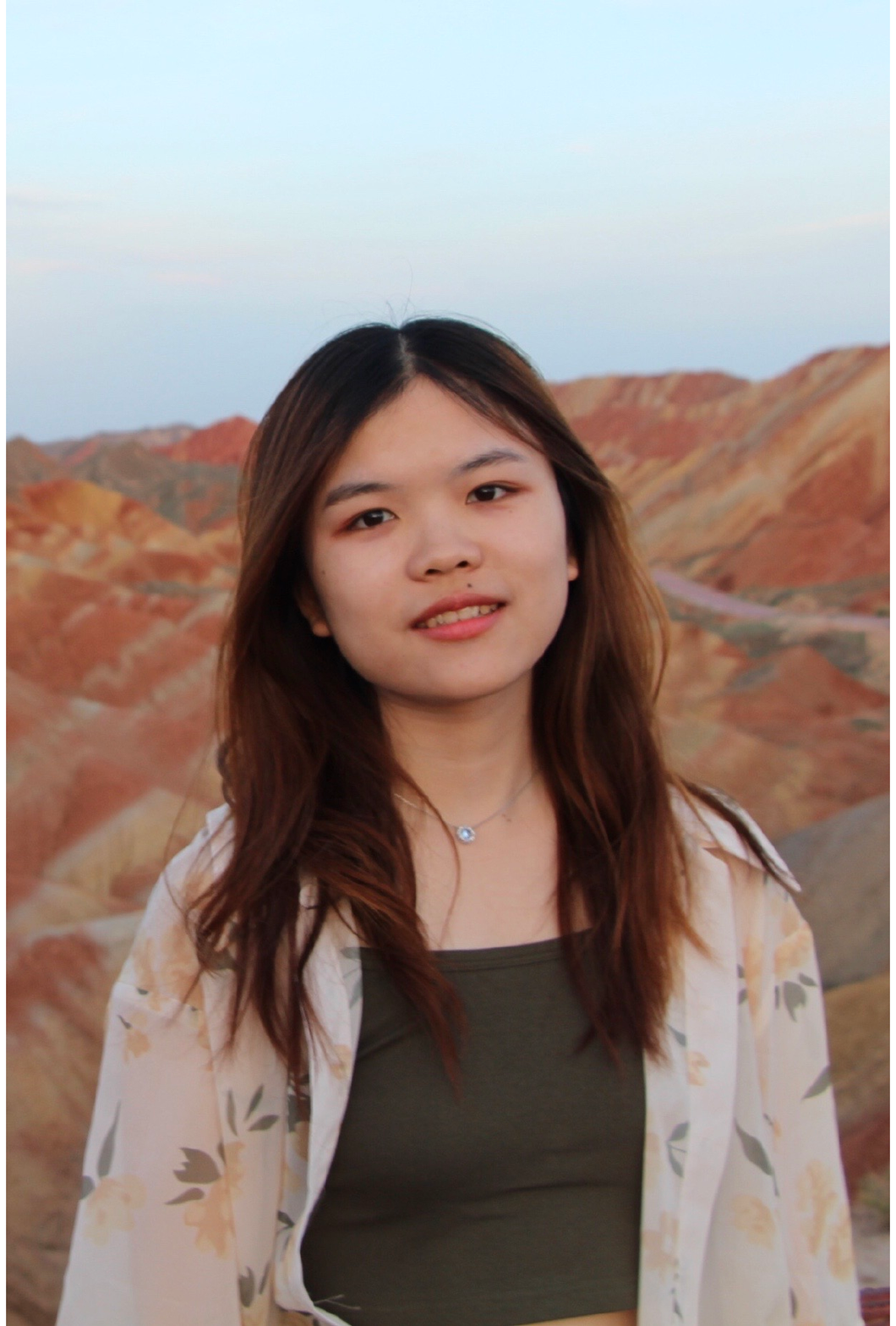}}]{Liuchun Yuan} received the B.S. degree from the Nanjing University of Information and Science Technology, China, in 2017. She is currently pursuing the M.S. degree with the School of Electronics and Information Technology, Sun Yat-sen University, Guangzhou, China. Her research interests mainly include computer vision and machine learning.\end{IEEEbiography}

\begin{IEEEbiography}[{\includegraphics[width=1in,height=1.25in,clip,keepaspectratio]{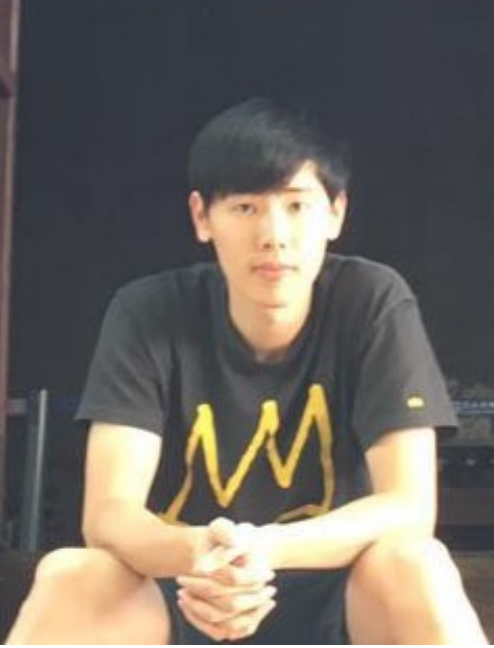}}]{JieFeng Peng} received his B.S. and M.E. degree from the Sun Yat-sen University, Guangzhou, China, in 2016 and 2019, respectively. His main research interests include deep learning and automated machine learning.\end{IEEEbiography}

\begin{IEEEbiography}[{\includegraphics[width=1in,height=1.25in,clip,keepaspectratio]{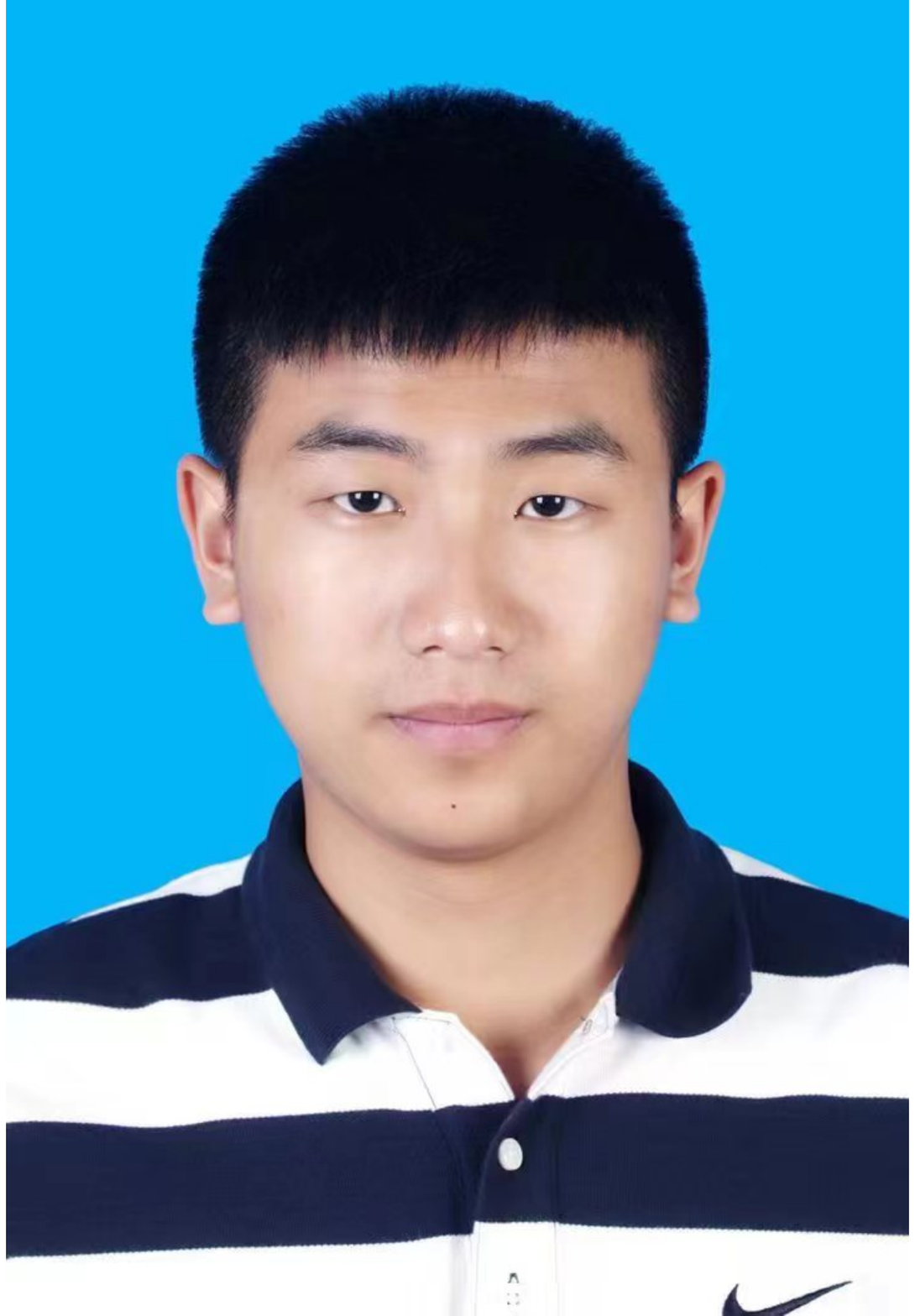}}]{Xiaoyu Xian} received the M.S. degrees from Beijing Jiaotong University, Beijing, China. Currently, he is a researcher with the technical department, CRRC Academy Co., Ltd., Beijing. His current research interests include optical character recognition and machine learning. \end{IEEEbiography}

\begin{IEEEbiography}[{\includegraphics[width=1in,height=1.25in,clip,keepaspectratio]{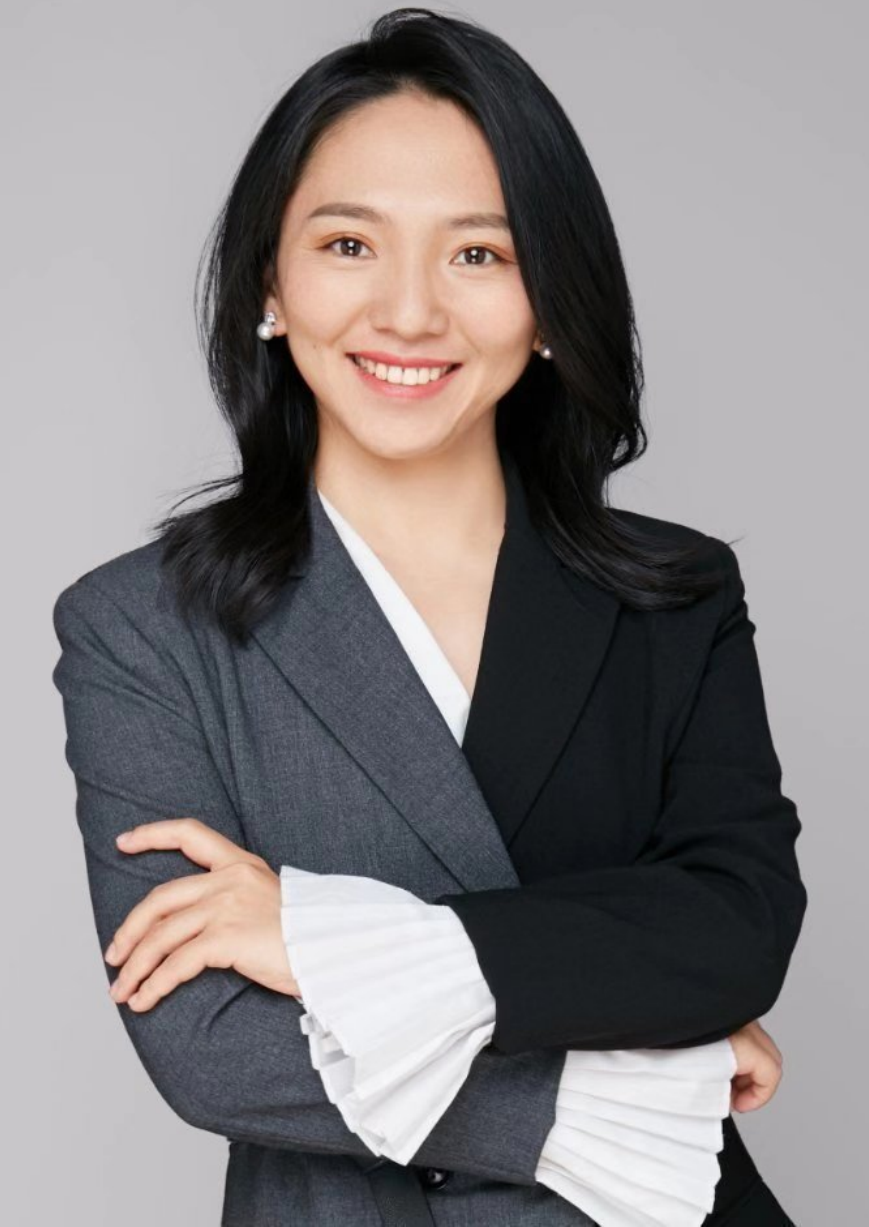}}]{Xiaodan Liang} is currently an Associate Professor at Sun Yat-sen University. She was a postdoc researcher in the machine learning department at Carnegie Mellon University, working with Prof. Eric Xing, from 2016 to 2018. She received her PhD degree from Sun Yat-sen University in 2016, advised by Liang Lin. She has published several cutting-edge projects on human-related analysis, including human parsing, pedestrian detection and instance segmentation, 2D/3D human pose estimation, activity recognition, dialogue system, and automated machine learning. More details could be found in her homepage: \url{https://lemondan.github.io}.\end{IEEEbiography}

\begin{IEEEbiography}[{\includegraphics[width=1in,height=1.25in,clip,keepaspectratio]{photo/XiaojunChang.pdf}}]{Xiaojun Chang} is a Professor at Australian Artificial Intelligence Institute, University of Technology Sydney. He is also an Honorary Professor at the School of Computing Technologies, RMIT University. Before joining UTS, he was an Associate Professor at the School of Computing Technologies, RMIT University, Australia. After graduation, he subsequently worked as a Postdoc Research Fellow at School of Computer Science, Carnegie Mellon University, Lecturer and Senior Lecturer in the Faculty of Information Technology, Monash University, Australia. He has spent most of his time working on exploring multiple signals (visual, acoustic, textual) for automatic content analysis in unconstrained or surveillance videos. He has achieved top performances in various international competitions, such as TRECVID MED, TRECVID SIN, and TRECVID AVS.\end{IEEEbiography}

\begin{IEEEbiography}[{\includegraphics[width=1in,height=1.25in,clip,keepaspectratio]{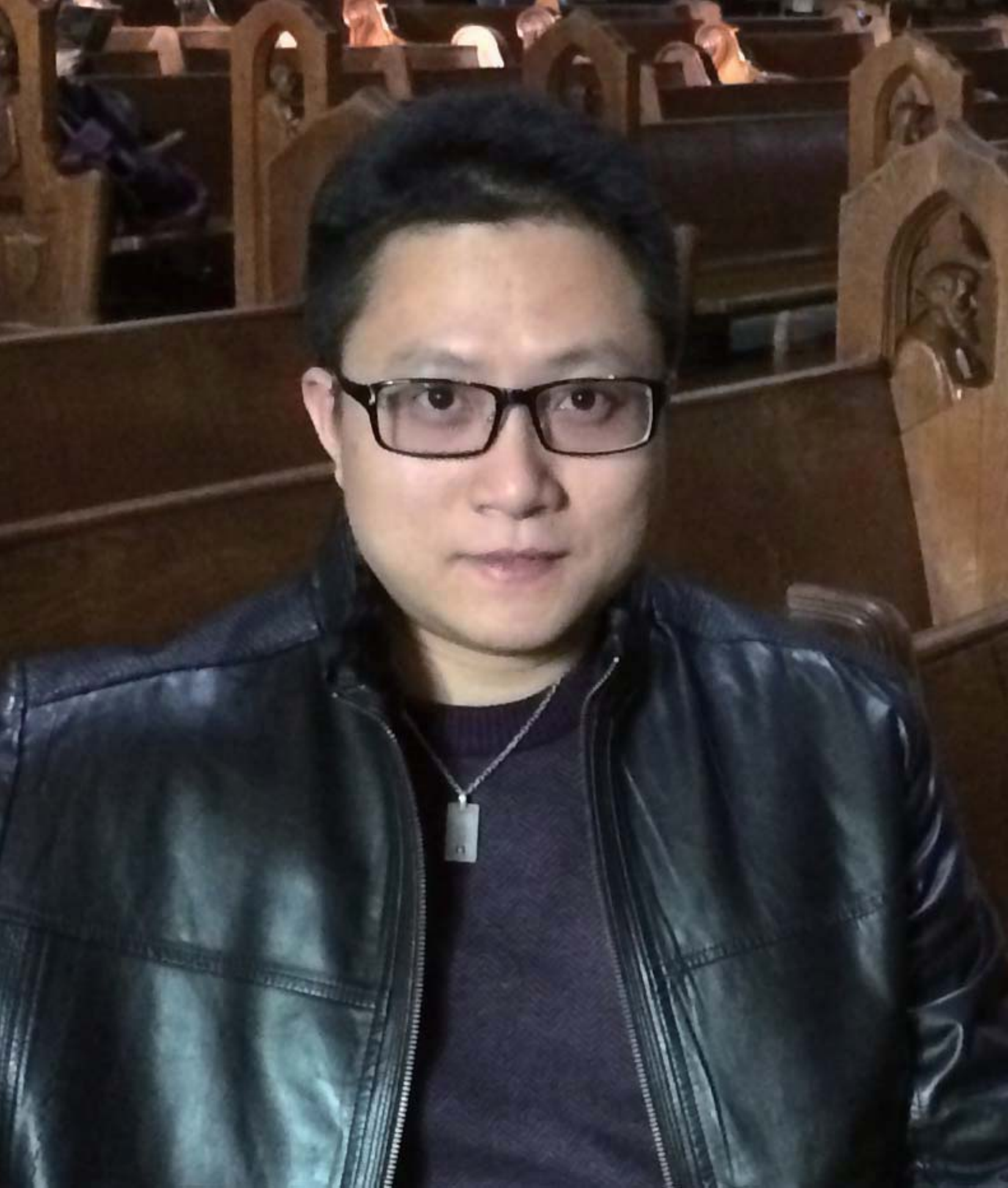}}]{Liang Lin} (M'09, SM'15) is a Full Professor of computer science at Sun Yat-sen University. He served as the Executive Director and Distinguished Scientist of SenseTime Group from 2016 to 2018, leading the R\&D teams for cutting-edge technology transferring. He has authored or co-authored more than 200 papers in leading academic journals and conferences (e.g., 20+ papers in TPAMI/IJCV), and his papers have been cited by more than 16,000 times. He is an associate editor of IEEE Trans. Neural Networks and Learning Systems and IEEE Trans. Human-Machine Systems, and served as Area Chairs for numerous conferences such as CVPR, ICCV, SIGKDD and AAAI. He is the recipient of numerous awards and honors including Wu Wen-Jun Artificial Intelligence Award, the First Prize of China Society of Image and Graphics, ICCV Best Paper Nomination in 2019, Annual Best Paper Award by Pattern Recognition (Elsevier) in 2018, Best Paper Dimond Award in IEEE ICME 2017, Google Faculty Award in 2012. His supervised PhD students received ACM China Doctoral Dissertation Award, CCF Best Doctoral Dissertation and CAAI Best Doctoral Dissertation. He is a Fellow of IAPR and IET.
\end{IEEEbiography}

\clearpage
\pagenumbering{arabic}
\section*{Appendix}

\begin{table*}
\setlength{\abovecaptionskip}{0.cm}
\setlength{\belowcaptionskip}{0.cm}
\begin{center}
\caption{\small{Detailed comparison of search space differences.}}\label{tab:searchspacecmp}
 \begin{tabular}{l||l|l|l|l|l|l|l}
 Methods & Granularity & Building block &Kernel size & Expand rate & SE ratio & Channels & Depth \\
 \hline\hline
 MnasNet \cite{Tan2018MnasNetPN} & block-level & MBBlock & 3, 5 & 3, 6 & 0, 0.25 & variable & variable \\
 MobileNetV3 \cite{Howard2019SearchingFM} & block-level & MBBlock & 3, 5 & 3, 6 & 0.25 & variable & variable \\
 EfficientNet \cite{Tan2019EfficientNetRM}  & block-level & MBBlock & 3, 5 & 3, 6 & 0.25 & variable & variable \\
 SPOS \cite{Guo2019SinglePO} & layer-level & ShuffleBlock & 3, 5, 7 & - & 0 & variable & variable \\
 ProxylessNAS \cite{Cai2018ProxylessNASDN} & layer-level & MBBlock & 3, 5, 7 & 3, 6 & 0 & fixed & variable \\
 FBNet-C \cite{Wu2018FBNetHE} & layer-level & MBBlock & 3, 5 & 1, 3, 6 & 0 & fixed & fixed \\
 FairNAS-A \cite{Chu2019FairNAS_arxiv} & layer-level & MBBlock & 3, 5, 7 & 3, 6 & 0 & fixed & fixed \\
 MoGA-A \cite{Chu2019MoGASB} & layer-level & MBBlock & 3, 5, 7 & 3, 6 & 0, 0.25 & fixed & fixed \\
 SCARLET-A \cite{Chu2019ScarletNASBT} & layer-level & MBBlock & 3, 5, 7 & 3, 6 & 0, 0.25 & fixed & variable \\
 PC-NAS-S \cite{anonymous2020improving} & layer-level & MBBlock & 3, 5, 7 & 1, 3, 6 & 0.25 & fixed & variable \\
 MixNet-M \cite{Tan2019MixConvMD}  & layer-level & MixConvBlock & mixed 3, 5, 7, 9, 11 & 3, 6 & 0.25 & fixed & variable \\
 DNA (ours) & layer-level & MBBlock & 3, 5, 7 & 3, 6 & 0.25 & variable & variable \\
 \end{tabular}
 \end{center}
 \vspace{-11pt}
\end{table*}

In the main body of this paper, we have presented the main idea in detail and have used prosperous, extensive, and fair experiments to demonstrate the effectiveness and efficiency of our DNA family. In this appendix, we present more details and experimental results to help readers better understand this paper. The contents of the appendix include proof of Theorem \ref{theorem:generalization}, algorithm boxes, details of our ImageNet NAS Bench, comparison of search spaces, and visualization of the searched architectures.

 \subsection*{Proof of Theorem \ref{theorem:generalization}}

We first rewrite Theorem \ref{theorem:generalization} into Theorem \ref{theorem:rewrite}.

\begin{theorem}\label{theorem:rewrite}
\textbf{(Generalization boundedness).} \emph{Let $\big\{ \psi_t^{\bullet}\big\}_{t=0}^T$ be the optimal network weights of the weight-sharing NAS, $\psi_0^{*}$ be the optimal network weights of the universal solution, and $\big\{ \psi_t^{*}\big\}_{t=1}^T$ be the optimal network weight of the stand-alone solution with $\mathcal{R} (\psi_t) = \tilde{\lambda} \big\| \psi_t\big\|_F^2$ where $\tilde{\lambda} = \frac{\lambda_0 \lambda}{\lambda_0 + \lambda}$. Then, for any candidate architecture $j$, the Frobenius norm of $\omega_j^{\bullet}$ is upper bounded by:}\begin{equation}\label{eqn:bound-a-rewrite}
\begin{aligned}
&\big\| \omega_j^{\bullet} \big\|_F \\
& \\
 = &\big\| \psi_0^{\bullet} + \psi_j^{\bullet} \big\|_F \\
 & \\
 \le &\sqrt{\frac{\lambda_0\big\| \psi_0^{*}\big\|_F^2}{\tilde{\lambda}} + \frac{\sum\limits_{t=1}^{T} \big| \mathcal{L}_{\alpha_t} (\psi_0^*) - \mathcal{L}_{\alpha_t} (\psi_t^*) \big| }{\tilde{\lambda} } },
\end{aligned}
\end{equation}\emph{where $\big| \mathcal{L}_{\alpha_t} (\psi_0^*) - \mathcal{L}_{\alpha_t} (\psi_t^*) \big|$ is the overall gap between the universal solution and the stand-alone solution.}
\end{theorem}

\begin{proof}
Since $\big\{ \psi_t^{\bullet} \big\}_{t=0}^T$ are the optimal network weights of the weight-sharing solution and the weight-sharing solution has lower losses than the universal one, we have:\begin{equation}\label{eqn:ineq-a}
\begin{aligned}
&\sum\limits_{t=1}^T \mathcal{L}_{\alpha_t} (\psi_0^{\bullet} + \psi_t^{\bullet}) + \lambda_0 \big\| \psi_0^{\bullet} \big\|_F^2 + \lambda \big\| \psi_t ^{\bullet}\big\|_F^2 \\
\le &\sum\limits_{t=1}^T\mathcal{L}_{\alpha_t}(\psi_0^*) + \lambda_0 \big\| \psi_0^*\big\|_F^2.
\end{aligned}
\end{equation}On the other hand, as $\big\{ \psi_t^*\big\}_{t=1}^T$ is the optimal network weights of the stand-alone solution and the stand-alone solution has lower losses than the weight-sharing one, for any candidate architecture $\alpha_j$, we have:\begin{equation}\label{eqn:ineq-b}
\begin{aligned}
&\mathcal{L}_{\alpha_j} (\psi_j^*) + \tilde{\lambda} \big\| \psi_j^* \big\|_F^2 \\
\le & \mathcal{L}_{\alpha_t} (\psi_0^{\bullet} + \psi_t^{\bullet}) + \tilde{\lambda} \big\| \psi_0^{\bullet} + \psi_j^{\bullet} \big\|_F^2\\
= & \mathcal{L}_{\alpha_t} (\psi_0^{\bullet} + \psi_t^{\bullet}) + \frac{1}{\lambda_0 + \lambda_t}\big\| \sqrt{\lambda_0} \sqrt{\lambda} \psi_0^{\bullet} + \sqrt{\lambda_0}\sqrt{\lambda}\psi_j^{\bullet} \big\|_F^2 \\
\le & \mathcal{L}_{\alpha_t} (\psi_0^{\bullet} + \psi_t^{\bullet}) + \frac{(\sqrt{\lambda_0})^2 + (\sqrt{\lambda})^2}{\lambda_0 + \lambda_t} \big\| \sqrt{\lambda_0}\psi_0^{\bullet}\big\|_F^2 \\
&+ \big\|\sqrt{\lambda}\psi_j^{\bullet} \big\|_F^2 \\
=& \mathcal{L}_{\alpha_t} (\psi_0^{\bullet} + \psi_t^{\bullet}) + \lambda_0\big\|\psi_0^{\bullet}\big\|_F^2 + \lambda\big\|\psi_j^{\bullet} \big\|_F^2
\end{aligned}
\end{equation}where the last inequality is due to the Cauchy-Schwarz inequality. Combining Eqn. \eqref{eqn:ineq-a} and \eqref{eqn:ineq-b} gives:\begin{equation}\label{eqn:ineq-c}
\begin{aligned}
&\mathcal{L}_{\alpha_t} (\psi_0^{\bullet} + \psi_t^{\bullet}) + \tilde{\lambda} \big\| \psi_0^{\bullet} + \psi_j^{\bullet} \big\|_F^2\\
\le & \mathcal{L}_{\alpha_t} (\psi_0^{\bullet} + \psi_t^{\bullet}) + \lambda_0\big\|\psi_0^{\bullet}\big\|_F^2 + \lambda\big\|\psi_j^{\bullet} \big\|_F^2 \\
\le & \sum\limits_{t=1}^T\mathcal{L}_{\alpha_t}(\psi_0^*) + \lambda_0 \big\| \psi_0^*\big\|_F^2 - \sum\limits_{t=1 \atop t\neq i}^{T}\mathcal{L}_{\alpha_t} (\psi_0^{\bullet} + \psi_t^{\bullet}) \\
\le & \sum\limits_{t=1}^T\mathcal{L}_{\alpha_t}(\psi_0^*) + \lambda_0 \big\| \psi_0^*\big\|_F^2 - \sum\limits_{t=1 \atop t\neq i}^{T} \mathcal{L}_{\alpha_t} (\psi_t^*)\\
= & \mathcal{L}_{\alpha_j} (\psi_j^*) + \lambda_0 \big\| \psi_0^*\big\|_F^2 + \sum\limits_{t=1}^T \big| \mathcal{L}_{\alpha_t} (\psi_0^*) - \mathcal{L}_{\alpha_t} (\psi_t^*) \big| \\
\le & \mathcal{L}_{\alpha_t} (\psi_0^{\bullet} + \psi_t^{\bullet}) + \lambda_0 \big\| \psi_0^*\big\|_F^2 + \sum\limits_{t=1}^T \big| \mathcal{L}_{\alpha_t} (\psi_0^*) - \mathcal{L}_{\alpha_t} (\psi_t^*) \big| ,
\end{aligned}
\end{equation}where the last inequality is because the weight-sharing solution has higher losses than the stand-alone one. Combining the first and the last line of Eqn. \eqref{eqn:ineq-c} gives:\begin{equation}
\big\| \psi_0^{\bullet} + \psi_j^{\bullet} \big\|_F \le \sqrt{\frac{\lambda_0 \big\| \psi_0^*\big\|_F^2 }{\tilde{\lambda}} + \frac{\sum\limits_{t=1}^T \big| \mathcal{L}_{\alpha_t} (\psi_0^*) - \mathcal{L}_{\alpha_t} (\psi_t^*) \big| }{\tilde{\lambda}} },
\end{equation} for any $j$.
\end{proof}

\textbf{Discussion:}
One could raise the point that this theorem is based on the assumption that ``weight-sharing solutions yield higher losses than the stand-alone ones," which may not always hold true as there are instances where models can derive benefits from weight-sharing training \cite{cai2019once}. We would like to provide further discussion to clarify this issue.

\textcolor{black}{
First and foremost, we would like to emphasize that even in the OnceForAll mode, Theorem \ref{theorem:generalization} still holds true. We know that in OnceForAll, two loss functions are used simultaneously for training the subnets: the classic cross-entropy loss and the KD loss that uses the predictions of the largest network for distillation. This technique is called inplace KD. In OnceForAll, the sub-models might not benefit from weight-sharing training, but they do benefit from inplace KD. Therefore, when considering retraining a stand-alone subnet of OnceForAll, it is advisable to utilize both cross-entropy and KD losses. By doing so, the loss of the stand-alone subnet still proves to be lower than that of the weight-sharing model, thereby affirming the validity of Theorem \ref{theorem:generalization}.
}

\textcolor{black}{Secondly, OnceForAll requires a relatively high similarity between the results of the student and the teacher, otherwise, different subnets would have optimization conflicts due to excessive structural diversity. Therefore, our DNA is more versatile and flexible compared to OnceForAll. For example, our DNA can be applied to cases where the teacher is CNN and the student is ViT.}

\textcolor{black}{Thirdly, it is important to note that the correlation between the accuracy of the subnets and their train-from-scratch accuracies remains unknown, rendering them somewhat opaque and leaving the upper bound of the subnets' capabilities uncertain. In contrast, our DNA method specifically targets and resolves the issue of inaccurate architecture ranking in weight-sharing NAS, thereby enhancing its efficiency and effectiveness. }

\textcolor{black}{
With approximately 6.4 million parameters, our DNA method achieves an impressive top-1 accuracy of 79.7\%, surpassing that of OnceForAll (76.4\%, see Table \ref{tab:imagenet}). These compelling results unequivocally demonstrate the consistent superiority of our method.
}

\subsection*{Algorithm boxes}

In Section \ref{sec:supernetdistill}, we mentioned that we designed two excellent algorithms in the rating step and the search step, respectively, to make our NAS very efficient. These two algorithms are the "feature sharing rating" and "Traversal search" algorithms, respectively. Here, we present two detailed algorithm boxes in Algorithm \ref{alg:eval} and Algorithm \ref{alg:search}. Hopefully, it can help readers to understand our algorithm well. The readers can also directly view our code to understand our algorithm (see the code link in the abstract, or \url{https://github.com/changlin31/DNA}).

\begin{algorithm}
\footnotesize
\caption{Feature sharing rating}\label{alg:eval}
\Indmm
\SetKw{Def}{define}
\SetKw{Out}{output}
\SetKwFunction{dfsforward}{DFS-Forward}
 \KwIn{Teacher's previous feature map $G_{prev}$, Teacher's current feature map $G_{curr}$, Root of the cell $Cell$, loss function $loss$}
 \KwOut{List of evaluation loss $L$}
\Indp
 \BlankLine
 \Def \dfsforward{$N$, $X$}:\\
\Indp
 $Y = N(X)$\;
 \eIf{$N$ has no $child$}{
  $append(L, loss(Y, G_{curr}))$\;}
 {
  \For{$C$ in $N.child$}{
   \dfsforward{$C$, $Y$}\;}}
\Indm
\BlankLine
\dfsforward{$Cell$, $G_{prev}$}\;
\Out{$L$}\;
\end{algorithm}

\setlength{\floatsep}{0.1cm}
\setlength{\textfloatsep}{0.1cm}
\begin{algorithm}
\footnotesize
\caption{Traversal search}\label{alg:search}
\Indmm
\KwIn{Block index $B$, the teacher's current feature map $G$, constrain $C$, model pool list $Pool$}
\KwOut{best model $M$}
\SetKw{Def}{define}
\SetKw{Out}{output}
\SetKwFunction{traverseblock}{SearchBlock}
\Indpp
 \Def \traverseblock{$B, size_{prev}, loss_{prev}$}:\\
\Indp
 \For{$i < length(Pool[B])$}{
 $size \leftarrow size_{prev} + size[i]$\;
 \If{$size > C$}{continue\;}
 $loss \leftarrow loss_{prev} + loss[i]$\;
 \eIf{B is last block}
 {
  \If{$loss \leq loss_{best}$}
   {$loss_{best} \leftarrow loss$\;
   $M \leftarrow$ index of each block
  }
  break\;}
 {\traverseblock{$B+1, size, loss$}\;}
}
\Indm
\BlankLine
\traverseblock{0}\;
\Out{$M$}\;
\end{algorithm}

\subsection*{Details of our ImageNet NAS Bench}

In Section \ref{sect:dataset}, we mentioned that we constructed a ImageNet NAS Bench to ensure fair and consistent evaluation and to compensate for the lack of full-resolution NAS benches on ImageNet. Here, we show some statistics about this benchmark so that the readers can better understand this benchmark in Fig. \ref{fig:bench}.

Fig. \ref{fig:bench}(a) presents the histogram of the top-1 accuracy of the architectures in our ImageNet NAS Bench. Fig. \ref{fig:bench} (b) summarize the top-1 accuracies \emph{w.r.t.} the parameter numbers. Fig. \ref{fig:bench} (c) summarize the top-1 accuracies \emph{w.r.t.} the computational complexity (i.e., FLOPs). 
These statistics truly reflect the real distribution of accuracy and computational complexity of different architectures in the real world, so hopefully, they will be useful to future researchers.

\begin{figure}[t]
\centering
\includegraphics[width=0.8\linewidth]{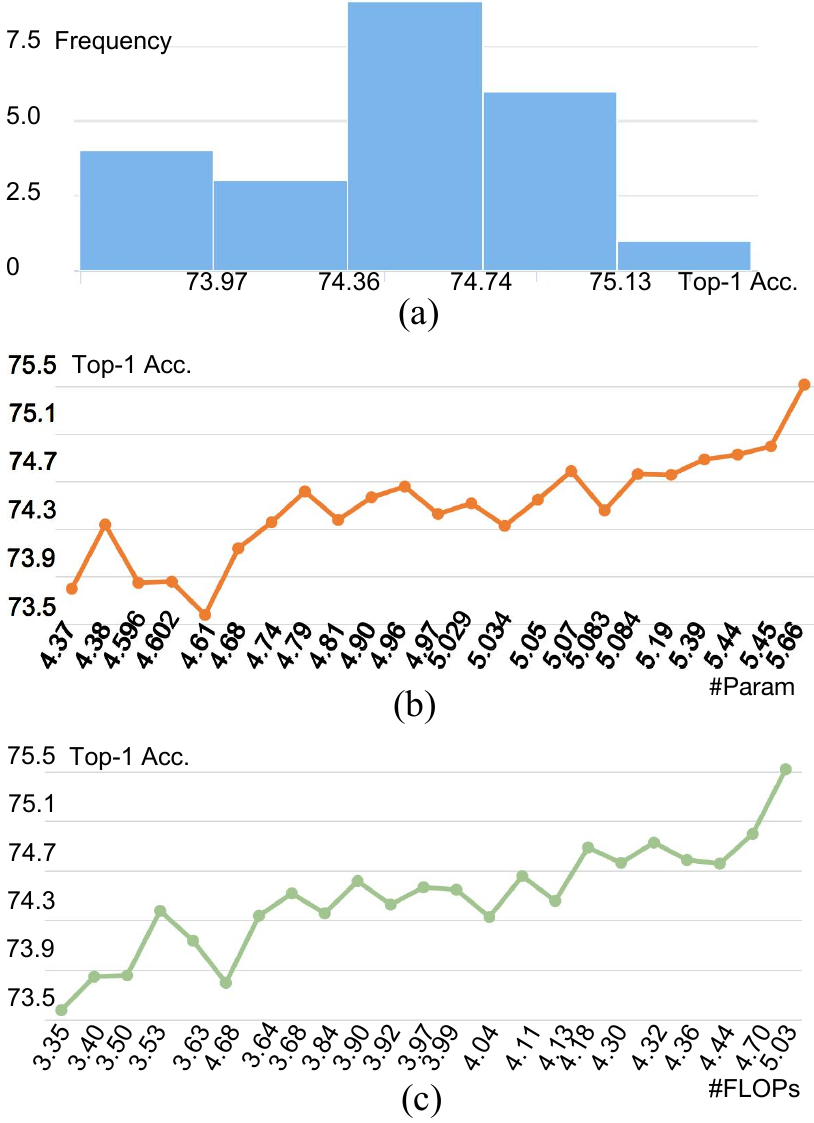}
\vspace{-11pt}
\caption{\small{Statistics of our ImageNet NAS Bench. (a) summarizes the histogram of the top-1 accuracies. (b) and (c) outline the top-1 accuracies \emph{w.r.t.} the parameter numbers and FLOPs.}}\label{fig:bench}
\end{figure}

\subsection*{Detailed comparison of search space differences}

As expected, the performance of searched models in NAS methods highly depends on the quality and variability of the search space.
To make a fair comparison, we only compare our DNA family with NAS methods using the MBConv search space in Table \ref{tab:imagenet}. This search space ignores NAS methods using cell-based search space, another critical NAS branch, in which ENAS \cite{Pham2018EfficientNA} and DARTS \cite{Liu2018DARTSDA} are typical examples. 
Note that we do not ignore these methods in Section \ref{sec:effctiveness} where the effectiveness of our method is sufficiently justified.

As mention in Section \ref{sec:search_spaces}, our MBConv search space is similar to most of the recent NAS works (\cite{Tan2018MnasNetPN, Tan2019EfficientNetRM, Chu2019ScarletNASBT, Chu2019MoGASB}) to ensure a fair comparison. Nevertheless, there still are some minor differences between these MobileNet block-based search spaces. The search spaces of all the NAS methods we compared in Table \ref{tab:imagenet} are shown in Table \ref{tab:searchspacecmp} in detail. As shown, our MBConv search space aligns with existing NAS works.

\subsection*{Architecture visualization}

Our searched architectures are visualized in Fig \ref{fig:arch_detail} in the appendix, from which we have several observations. \textbf{i)} Searched under no constraint, DNA-d tends to choose relatively expensive operations with high expansion rate and large kernel size, thus achieving the best performance. This verifies that our DNA is able to find optimal architectures in the search space. \textbf{ii)} Under the constraint of maximum parameter number, DNA-c tends to discard the operations with redundant channels to save the parameters. It also tends to select a lower expansion rate in the last blocks since the last blocks have more channels than the first blocks. \textbf{iii)} Under the constraint of maximum computational cost, DNA-b and DNA-a tend to select operations with fewer channels and lower expansion rate evenly in each block. The significant style difference of the high-performing architectures in Fig \ref{fig:arch_detail} proves DNA's architecture search ability.

\begin{figure*}
\centering
     \includegraphics[width=0.9\linewidth]{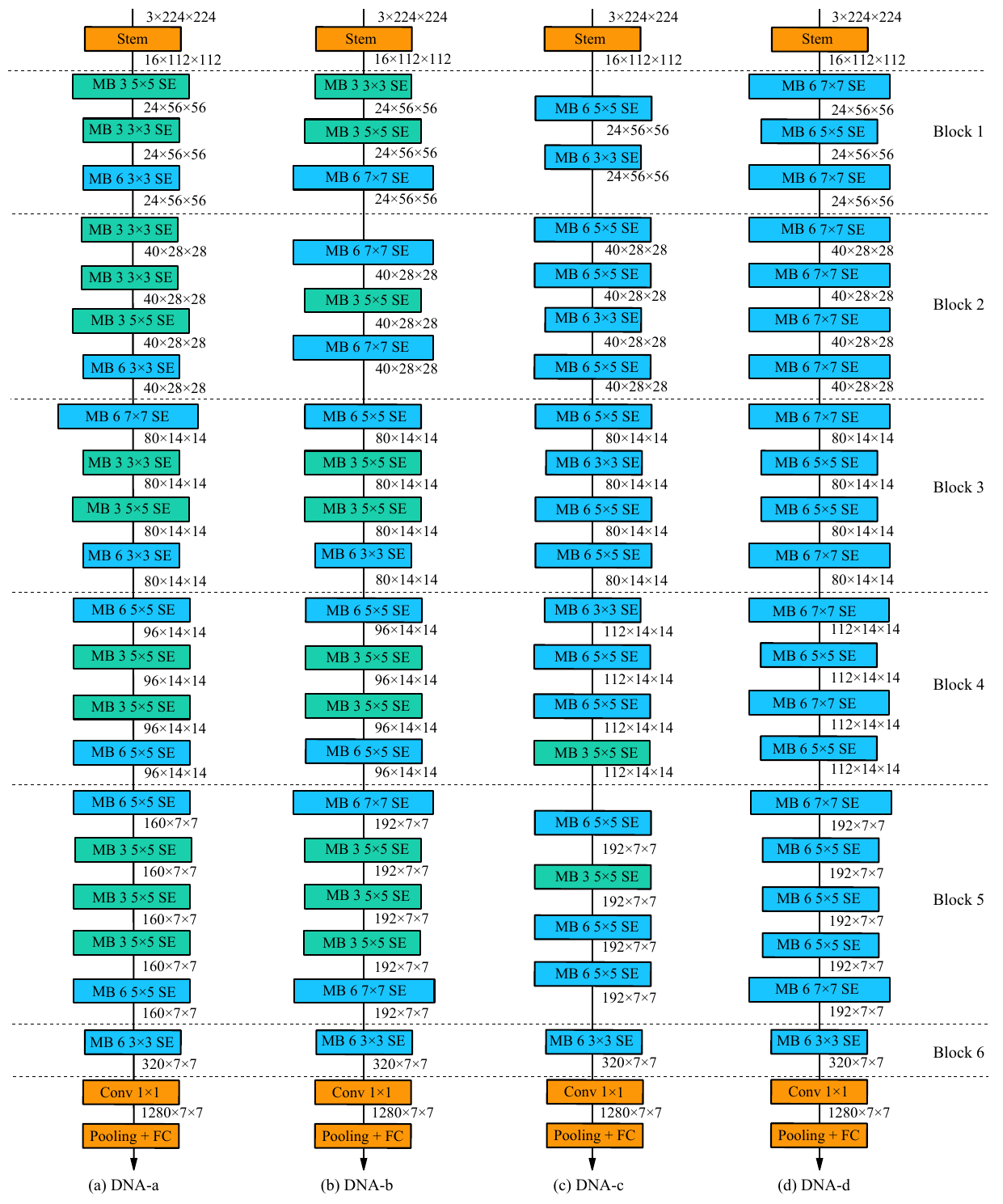}
     \caption{\small{Architectures of DNA-a,b,c,d. `MB $x$ $y\times y$'represents an inverted-bottleneck-convolution with expand rate $x$ and kernel size $y$.}}\label{fig:arch_detail}
\end{figure*}

\ifCLASSOPTIONcaptionsoff
 \newpage
\fi
\end{document}